\documentclass{article} % For LaTeX2e
\usepackage{iclr2026_conference,times}
\iclrfinalcopy
\usepackage{amsmath,amsfonts,bm}

\def\eqref#1{equation~\ref{#1}}
\def\1{\bm{1}}

\DeclareMathAlphabet{\mathsfit}{\encodingdefault}{\sfdefault}{m}{sl}
\SetMathAlphabet{\mathsfit}{bold}{\encodingdefault}{\sfdefault}{bx}{n}

\usepackage{hyperref}
\usepackage{url}
\usepackage{amsthm}
\usepackage{amsmath}
\usepackage{mathtools}
\usepackage{amsfonts} % Optional, for math fonts
\usepackage{amssymb}  % Optional, for extra symbols
\usepackage{wrapfig}
\usepackage{natbib}
\usepackage{subcaption}
\usepackage{tabularx}

\theoremstyle{plain}
\newtheorem{theorem}{Theorem}[section]
\newtheorem{proposition}[theorem]{Proposition}
\newtheorem{lemma}[theorem]{Lemma}

\newtheorem{example}{Example}[section]
\theoremstyle{definition}
\newtheorem{definition}[theorem]{Definition}

\theoremstyle{remark}
\newtheorem{remark}[theorem]{Remark}
\DeclareMathOperator{\tr}{Tr}

\usepackage{caption}
\usepackage[utf8]{inputenc} % allow utf-8 input
\usepackage[T1]{fontenc}    % use 8-bit T1 fonts
\usepackage{hyperref}       % hyperlinks
\usepackage[capitalize]{cleveref} 
\usepackage{url}            % simple URL typesetting
\usepackage{booktabs}       % professional-quality tables
\usepackage{arydshln}   % Required for dashed lines
\usepackage{siunitx}    % optional for number formatting
\usepackage{makecell}   % for stacked cells if needed
\usepackage{nicefrac}       % compact symbols for 1/2, etc.
\usepackage{microtype}      % microtypography
\usepackage[dvipsnames]{xcolor}         % colors
\usepackage{etoolbox} % Add this package
\usepackage{booktabs} 
\usepackage{multirow}
\usepackage{array}
\usepackage[pdftex]{graphicx}
\definecolor{brickred}{rgb}{0.8, 0.25, 0.33}

\newcommand{\Phinet}{\Phi}

\begin{document}

\title{Equivariant Neural Networks for General \\ Linear Symmetries on Lie Algebras}
\maketitle
\vspace{-5em}
% \begin{flushleft}
\textbf{Chankyo Kim}\textsuperscript{1,*} \quad
\textbf{Sicheng Zhao}\textsuperscript{1,*} \quad
\textbf{Minghan Zhu}\textsuperscript{1,2} \quad
\textbf{Tzu-Yuan Lin}\textsuperscript{3} \quad
\textbf{Maani~Ghaffari}\textsuperscript{1} \\
% \vspace{0.8em}
% 소속 기관 목록
\noindent
\textsuperscript{1}University of Michigan\\
\textsuperscript{2}University of Pennsylvania\\
\textsuperscript{3}Massachusetts Institute of Technology\\
% \end{flushleft}
\texttt{\{chankyo,sichengz,minghanz,maanigj\}@umich.edu}, 
\texttt{tzuyuan@mit.edu}
\renewcommand{\thefootnote}{}%
\footnotetext{* Equal contribution.}%
\addtocounter{footnote}{-1}

\begin{abstract}
% Encoding symmetries is a powerful inductive bias for improving the generalization of deep neural networks. However, most existing equivariant models are limited to simple symmetries like rotations, failing to address the broader class of general linear transformations, $\mathrm{GL}(n)$, that appear in many scientific domains. We introduce \textbf{Reductive Lie Neurons (ReLNs)}, a novel neural network architecture exactly equivariant to these general linear symmetries. ReLNs are designed to operate directly on a wide range of structured inputs, including general $n$-by-$n$ matrices.ReLNs introduce a novel adjoint-invariant bilinear layer to achieve stable equivariance for both Lie-algebraic features and matrix-valued inputs, \textit{without requiring redesign for each subgroup}. This architecture overcomes the limitations of prior equivariant networks that only apply to compact groups or simple vector data. We validate ReLNs' versatility across a spectrum of tasks: they outperform existing methods on algebraic benchmarks with $\mathfrak{sl}(3)$ and $\mathfrak{sp}(4)$ symmetries and achieve competitive results on a Lorentz-equivariant particle physics task. In 3D drone state estimation with geometric uncertainty, ReLNs jointly process velocities and covariances, yielding significant improvements in trajectory accuracy. ReLNs provide a practical and general framework for learning with broad linear group symmetries on Lie algebras and matrix-valued data.

Many scientific and geometric problems exhibit general linear symmetries, yet most equivariant neural networks are built for compact groups or simple vector features, limiting their reuse on matrix-valued data such as covariances, inertias, or shape tensors.
We introduce \textbf{Reductive Lie Neurons (ReLNs)}, an exactly $\mathrm{GL}(n)$-equivariant architecture that natively supports matrix-valued and Lie-algebraic features.
ReLNs resolve a central stability issue for reductive Lie algebras by introducing a non-degenerate adjoint (conjugation)-invariant bilinear form, enabling principled nonlinear interactions and invariant feature construction in a single architecture that \textit{transfers across subgroups without redesign}.
We demonstrate ReLNs on algebraic tasks with $\mathfrak{sl}(3)$ and $\mathfrak{sp}(4)$ symmetries, Lorentz-equivariant particle physics, uncertainty-aware drone state estimation via joint velocity--covariance processing, learning from 3D Gaussian-splat representations, and EMLP double-pendulum benchmark spanning multiple symmetry groups. ReLNs consistently match or outperform strong equivariant and self-supervised baselines while using substantially fewer parameters and compute, improving the accuracy–efficiency trade-off and providing a practical, reusable backbone for learning with broad linear symmetries.
Project page: \url{https://reductive-lie-neuron.github.io/}

\end{abstract}

\section{Introduction}
Exploiting symmetries in data is a fundamental principle of geometric deep learning. By enforcing equivariance, where model outputs transform predictably with inputs, neural networks can leverage strong inductive biases to improve data efficiency and generalization.

Substantial progress has been made for compact symmetry groups such as rotations $\mathrm{SO}(3)$ and Euclidean isometries $\mathrm{E}(n)$~\citep{bronstein2021geometric, cohen2016group, thomas2018tensor, satorras2021n, geiger2022e3nn}. However, many geometric and physical systems exhibit broader \emph{linear} symmetries that are naturally expressed by the non-compact general linear group \mbox{$\mathrm{GL}(n)=\{A\in\mathbb{R}^{n\times n}:\det(A)\neq 0\}$}. Despite its ubiquity, general-purpose, exact equivariant architectures for $\mathrm{GL}(n)$ remain comparatively underdeveloped. Appendix~\ref{app:groups_summary} summarizes these groups and applications.

A key challenge for $\mathrm{GL}(n)$-equivariant learning is non-compactness. Unlike compact groups (e.g., $\mathrm{SO}(3)$), $\mathrm{GL}(n)$ does not come with a canonical $\mathrm{Ad}$-invariant positive-definite inner product: the Killing form is degenerate on $\mathfrak{gl}(n)$, and one loses the norm-preserving (unitary/orthogonal) structure that standard steerable irrep-based models rely on in practice~\citep{thomas2018tensor,geiger2022e3nn}. At the Lie-algebra level, $\mathfrak{gl}(n)=\mathbb{R}^{n\times n}$ is reductive but \emph{not} semisimple, so this degeneracy removes a canonical invariant metric and complicates stable invariant scalars for equivariant nonlinear. As a result, existing approaches either restrict attention to semisimple algebras or rely on ad hoc choices that do not provide a general, stable recipe for $\mathrm{GL}(n)$ adjoint equivariance.

Beyond this algebraic obstacle, real-world geometric inputs are often \emph{heterogeneous}, with different attributes obeying different transformation laws. Vectors transform by a left action (e.g., $v\mapsto Av$), whereas matrix-valued tensors such as uncertainty, inertia, or shape covariances transform by congruence (e.g., $\Sigma\mapsto A\Sigma A^\top$). Learning to \emph{jointly} couple vector--matrix channels in a symmetry-consistent and numerically stable way remains challenging: naive flattening discards tensor structure, and eigen/spectral parameterizations can be ambiguous and fragile under differentiation~\citep{magnus1985differentiating}. These issues arise in robotics state estimation, where velocities and covariance uncertainty must be processed together, and in 3D Gaussian Splatting, which couples a mean $\mu\in\mathbb{R}^3$ with an anisotropic covariance $\Sigma\in\mathrm{SPD}(3)$.

In this paper, we introduce \emph{Reductive Lie Neurons (ReLNs)}, a general-purpose architecture that is \emph{exactly} equivariant to the adjoint action of $\mathrm{GL}(n)$ on $\mathfrak{gl}(n)$. ReLNs are built around a non-degenerate $\mathrm{Ad}$-invariant bilinear form that resolves Killing-form degeneracy on reductive algebras and enables stable invariant gating/normalization and expressive equivariant nonlinearities within a single backbone. We evaluate ReLNs on algebraic benchmarks ($\mathfrak{sl}(3)$, $\mathfrak{sp}(4)$), uncertainty-aware drone state estimation with joint velocity--covariance processing, learning from 3D Gaussian-splat representations with comparisons to Gaussian-MAE, and the EMLP double-pendulum benchmark; additional Lorentz-equivariant results are deferred to the appendix.

Our main contributions are:
\begin{enumerate} 

    \item We propose \textbf{Reductive Lie Neurons (ReLNs)}, a novel, general-purpose architecture for exact $\mathrm{GL}(n)$ adjoint equivariance, built upon a $\mathrm{Ad}$-invariant bilinear form that resolves the degeneracy issues of the standard Killing form on reductive algebras. %, built upon a learnable $\mathrm{Ad}$-invariant bilinear form.
    % that overcomes the limitations of the degenerate Killing form.
    
    % \item We unify classical left-action equivariance with our adjoint-action framework. Through a equivariant embedding map, we show that problems defined with left-action symmetries—such as particle physics with Lorentz group symmetry and 3D point clouds with SO(3) symmetry—can be solved within our unified architecture without specialized designs.

    \item We introduce a unified lifting that embeds heterogeneous geometric inputs, left-acting vectors and congruence-transforming matrices (e.g., covariances, 3D Gaussian parameters), into a shared $\mathfrak{gl}(n)$ feature space, enabling a \textbf{single adjoint-equivariant backbone} with exact equivariance.

    % \item We establish the framework for geometric uncertainty-aware equivariant learning, enabling models to treat matrix-valued data that transforms under congruence (e.g., covariance tensors) as geometric objects.
    % \item We demonstrate the effectiveness of ReLNs through extensive experiments, showing that they outperform prior methods on Lie-algebraic benchmarks and achieve significant improvements in accuracy and robustness on a challenging 3D drone state estimation task.

    \item We demonstrate \textbf{practical utility and efficiency} across diverse regimes, matching Lie-algebra baselines on $\mathfrak{sl}(3)$ and $\mathfrak{sp}(4)$, achieving strong robustness gains on uncertainty-aware drone state estimation and 3D Gaussian-splat representation learning, substantially reducing per-step compute compared to EMLP on the double-pendulum benchmark, matching baseline performance on jet top-tagging with Lorentz-equivariant embeddings.

\end{enumerate}

\begin{figure}[t]
  \centering
  % Use PDF/SVG for vector quality; adjust width as needed (e.g., 0.75\linewidth)
  \includegraphics[width=\linewidth]{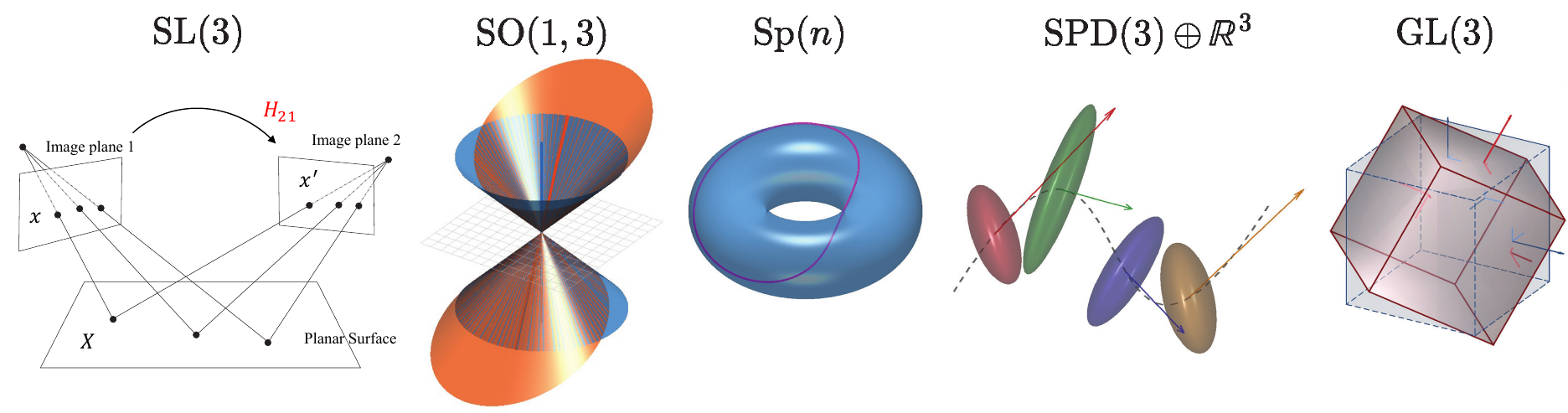}

    \caption{Examples of Lie groups and related manifolds in scientific applications. From left: the special linear group \(\mathrm{SL}(3)\) (image homography), the Lorentz group \(\mathrm{SO}^+(1,3)\) (spacetime symmetry), symplectic groups \(\mathrm{Sp}(n)\) (Hamiltonian mechanics), the \(\mathrm{SPD}(3) \times \mathbb{R}^{3}\) state space (probabilistic estimation), and the general linear group \(\mathrm{GL}(3)\) (modeling stress-strain in continuum mechanics).}
  \label{fig:domain_groups}
\end{figure}

\begin{figure}[t]
    \centering
    \includegraphics[width=0.7\linewidth]{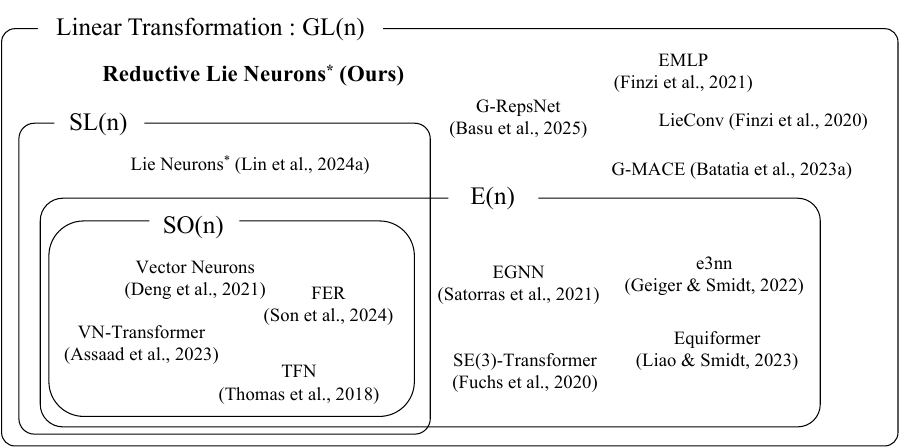}
    \caption{A taxonomy of selected representative equivariant neural architectures, categorized by the symmetries to which they are equivariant. This diagram situates our work, ReLNs, among other notable methods that are often specialized for subgroups such as $\mathrm{SL}(n)$, $\mathrm{SO}(n)$, or the Euclidean group $\mathrm{E}(n)$. An asterisk ($^*$) denotes methods equivariant to the group’s adjoint action.}
    \label{fig:taxonomy}
\end{figure}

\section{Related Work}
Encoding symmetry into neural architectures is a well-established inductive bias for improving data efficiency and generalization~\citep{bronstein2021geometric}. Most geometric deep learning work targets Euclidean isometries such as rotations and rigid motions (e.g., $\mathrm{SO}(n)$ and $\mathrm{SE}(n)$). For grid data, representative approaches include group-equivariant and steerable CNNs~\citep{cohen2016group, weiler2018learning, weiler2021general}. For sets and graphs, many methods build features from irreducible representations and tensorial message passing, including TFNs, $\mathrm{E}(n)$-GNNs, and equivariant transformers~\citep{thomas2018tensor, fuchs2020se, satorras2021n, batatia2022mace, liao2022equiformer, assaad2022vn, hutchinson2021lietransformer}, alongside lighter vector-based variants~\citep{deng2021vector, son2024intuitive}. Related theoretical work studies universality of invariant/equivariant architectures~\citep{maron2019universality}, and model-agnostic strategies such as canonicalization and frame averaging provide complementary routes to equivariance~\citep{puny2021frame, lin2024equivariance, kaba2023equivariance, panigrahi2024improved}. While several frameworks can represent higher-order tensors (e.g., TFN~\citep{thomas2018tensor} and EMLP~\citep{finzi2021practical}), scalable architectural design for \emph{general matrix-valued} quantities (e.g., covariances transforming as $\Sigma \mapsto R\Sigma R^\top$) remains less standardized in practice.

\textbf{Non-compact groups and $\mathrm{GL}(n)$.}
Extending equivariance to non-compact symmetries remains an active area, notably exemplified by Lorentz-equivariant architectures that operate on Minkowski geometry in particle physics \citep{bogatskiy2020lorentz, gong2022efficient, batatia2023a, zhdanov2024clifford}. Complementary research utilizes geometric (Clifford) algebras to represent features as multivectors, leading to frameworks such as Geometric Algebra Transformers (GATr) \cite{brehmer2023geometric}, Clifford-equivariant simplicial message passing \cite{liu2024clifford}, and lightweight designs like GLGENN \cite{filimoshina2025glgenn}. Approaches to non-compact groups include generalizations of group convolutions and kernels~\citep{xu2022unified, helwig2023group}, constructions based on matrix functions or reductive groups~\citep{batatia2023equivariant, batatia2023a}, Lie-algebra-based kernels and decompositions~\citep{finzi2020generalizing, mironenco2024lie, shumaylov2025lie}, differential-operator and differential-invariant methods~\citep{he2022neural, shen2020pdo, jenner2022steerable, sangalli2022differential, li2024affine}, and algebraic constraint-solving frameworks such as EMLP and G-RepsNet~\citep{finzi2021practical, basu2024g}. These lines provide general tools but can incur substantial overhead (e.g., integration/basis computation) or lack inductive biases such as locality. Within Lie-algebraic adjoint-equivariant learning, Lie Neurons~\citep{lin2023lie} establish $\mathrm{Ad}$-equivariant neural network but rely on semisimple algebras where the Killing form is non-degenerate, and thus do not directly cover reductive algebras. This gap is relevant when inputs include general matrix-valued quantities that transform by adjoint action. Classical invariant filtering explicitly respects these transformation rules~\citep{barrau2016invariant, hartley2020contact}, but is model-based and does not provide a learned backbone.

\textbf{Learning on 3D Gaussian splats.}
3D Gaussian Splatting represents geometry using anisotropic primitives parameterized by a mean $\mu \in \mathbb{R}^3$ and covariance $\Sigma \in \mathrm{SPD}(3)$. Recent learning frameworks over such primitives, including self-supervised learning and dynamics/segmentation models~\citep{ma2025large, lu2024manigaussian, zhang2024dynamics, ye2024gaussian}, often treat heterogeneous attributes as loosely coupled channels (e.g., concatenating $\mu$ with rotation parameters and applying generic architectures), which does not explicitly enforce the coupled transformation of $(\mu,\Sigma)$. This decoupled treatment disregards the underlying geometry, representing vectors and tensors as mixed latent feature that fail to recognize their transformation laws. Consequently, such models struggle to capture the symmetry of the scene without massive data augmentation. ReLN formulation targets this gap by embedding these heterogeneous attributes into a unified Lie-algebraic space, enforcing the coupled equivariant structure of $(\mu, \Sigma)$, exploiting the isomorphism between congruence and adjoint actions under $\mathrm{SO(3)}$.

ReLNs overcome these limitations by establishing exact adjoint equivariance on $\mathfrak{gl}(n)$ via a non-degenerate $\mathrm{Ad}$-invariant bilinear form. This approach bypasses the reliance on degenerate invariants and group-specific integration required by previous methods. \autoref{fig:taxonomy} situates ReLNs relative to prior equivariant frameworks.

\section{Preliminaries}
\label{sec:preliminaries}
We build equivariant networks on a reductive Lie algebra $\mathfrak{gl}(n)$, requiring equivariance under the adjoint action $\mathrm{Ad}_g(X)=gXg^{-1}$ for $g \in \mathrm{GL}(n)$. A key obstruction is that $\mathfrak{gl}(n)$ is reductive but not semisimple: its canonical  $\mathrm{Ad}$-invariant \emph{symmetric bilinear form} (the Killing form) is degenerate. Many equivariant architectures rely on a \emph{non-degenerate} invariant form to produce invariant scalars (e.g., gating/normalization) and to parameterize genuinely nonlinear equivariant operations (e.g., Vector Neurons~\citep{deng2021vector}, Lie Neurons~\citep{lin2023lie}); degeneracy makes these constructions ill-conditioned or effectively linear. We resolve this by introducing a non-degenerate  $\mathrm{Ad}$-invariant bilinear form on $\mathfrak{gl}(n)$, enabling fully nonlinear adjoint-equivariant layers. Our scope is reductive Lie algebras; extending to non-reductive cases (e.g., $\mathfrak{aff}(n)$) generally requires additional non-canonical completions. For details of Lie theory and background, see Appendix~\ref{app:prelim}.

\section{Reductive Lie Neurons: Architecture}
\label{sec:architecture}
We present ReLNs, a framework for building deep networks equivariant to the adjoint action of $\mathrm{GL}(n)$ on its Lie algebra $\mathfrak{gl}(n)$. The design centers on a non-degenerate, $\mathrm{Ad}$-invariant bilinear form on the reductive Lie algebra $\mathfrak{gl}(n)$ and a complete toolbox of equivariant linear maps, nonlinearities, pooling, and invariant readouts. 

\subsection{An $\mathrm{Ad}$-invariant bilinear form for reductive Lie algebras}
\label{sec:bilinear_form}
A key obstruction to extending Lie-algebraic designs (e.g., Lie Neurons~\citep{lin2023lie}) from semisimple Lie algebras to $\mathfrak{gl}(n)$ is that the Killing form may be \emph{degenerate} on non-semisimple algebras. A finite-dimensional Lie algebra is \emph{reductive} if it admits an ideal decomposition
$\mathfrak g=\mathfrak z(\mathfrak g)\oplus[\mathfrak g,\mathfrak g]$ with $[\mathfrak g,\mathfrak g]$ semisimple, and the Killing form is
$B_{\mathfrak g}(X,Y)=\tr(\mathrm{ad}_X\circ\mathrm{ad}_Y)$; see Appendix~\ref{app:reductive} for detailed definitions and the Killing form on $\mathfrak g$.
In particular, $\mathfrak{gl}(n)$ is reductive but not semisimple: it contains a nontrivial center on which $\mathrm{ad}$ vanishes, forcing degeneracy of the Killing form. To obtain well-posed invariant contractions and nonlinearities, we construct an $\mathrm{Ad}$-invariant \emph{non-degenerate completion} of the Killing form on any reductive Lie algebra.

\begin{definition}[Modified Bilinear Form on a Reductive Lie Algebra]
\label{def:mod-killing}
If $\mathfrak g$ is reductive, then $\mathfrak g=\mathfrak z(\mathfrak g)\oplus[\mathfrak g,\mathfrak g]$, where $\mathfrak z(\mathfrak g)$ is the center. Choose any $\mathrm{Ad}$-invariant inner product $\langle\cdot,\cdot\rangle_{\mathfrak z}$ on $\mathfrak z(\mathfrak g)$ (for connected $G$ this is automatic since $\mathrm{Ad}|_{\mathfrak z(\mathfrak g)}:G \to \mathrm{GL}(\mathfrak z(\mathfrak g))$ is locally constant. See Appendix~\ref{app:adjointrep} for the formal definition of \(\mathrm{Ad}\)) . For $Z_i\in\mathfrak z(\mathfrak g)$ and $X_i\in[\mathfrak g,\mathfrak g]$ define %
\begin{equation}
    \widetilde B(Z_1{+}X_1,\,Z_2{+}X_2) :=
\langle Z_1,Z_2\rangle_{\mathfrak z}+B(X_1,X_2),
\qquad 
\end{equation}
where $B$ denotes the Killing form on the semisimple ideal $[\mathfrak g,\mathfrak g]$.
\end{definition}

\begin{proposition} [Non-degeneracy and $\mathrm{Ad}$-invariance]
\label{prop:mod-killing}
The form $\widetilde B$ is symmetric, $\mathrm{Ad}$-invariant, and non-degenerate.
Moreover, $\mathfrak z(\mathfrak g)$ and $[\mathfrak g,\mathfrak g]$ are $\widetilde B$-orthogonal, with
$\widetilde B|_{[\mathfrak g,\mathfrak g]}=B$ and $\widetilde B|_{\mathfrak z(\mathfrak g)}=\langle\cdot,\cdot\rangle_{\mathfrak z}$.
\end{proposition}

\begin{proof}[Proof sketch]
$B$ vanishes on $\mathfrak z(\mathfrak g)$ and is $\mathrm{Ad}$-invariant; by construction $\langle\cdot,\cdot\rangle_{\mathfrak z}$ is $\mathrm{Ad}$-invariant. Symmetry is immediate. Nondegeneracy follows since $B$ is nondegenerate on the semisimple ideal and $\langle\cdot,\cdot\rangle_{\mathfrak z}$ is nondegenerate on the center; orthogonality holds because $B(\mathfrak z(\mathfrak g),[\mathfrak g,\mathfrak g])=0$. Detailed proofs are provided in Appendix~\ref{app:bilinear_details}.
\end{proof}
% $B$ is $\mathrm{Ad}$-invariant and vanishes on $\mathfrak z(\mathfrak g)$. By $\mathrm{Ad}$-invariance of $\langle\cdot,\cdot\rangle_{\mathfrak z}$, $\widetilde B$ is $\mathrm{Ad}$-invariant. 
% Symmetry is clear from the definition. 
% For nondegeneracy, if $\widetilde B(Z{+}X,\cdot)\equiv0$, then $B(X,\cdot)=0$ on $[\mathfrak g,\mathfrak g]$, 
% hence $X=0$ since the Killing form $B$ is nondegenerate on the semisimple ideal, and 
% $\langle Z,\cdot\rangle_{\mathfrak z}=0$ implies $Z=0$.
% Orthogonality follows because $B(\mathfrak z(\mathfrak g),[\mathfrak g,\mathfrak g])=0$.

For our primary case $\mathfrak g=\mathfrak{gl}(n)=\mathbb RI \oplus \mathfrak{sl}(n)$, we use the closed form
\begin{equation}
\widetilde B(X,Y) \;=\; 2n\,\tr(XY)\;-\;\tr(X)\tr(Y).
\label{eq:killing-gl-concrete}
\end{equation}
This $\widetilde B$ is our basic $\mathrm{Ad}$-invariant contraction throughout ReLN architecture.
% Additional discussion (including its relation to the classical Killing form and the $\mathfrak{so}(3)\simeq\mathbb R^3$ correspondence) is deferred to Appendix~\ref{app:connections}.

\paragraph{Verification and Relation to Prior Bilinear Forms.}
Our concrete form in Eq.~\ref{eq:killing-gl-concrete} satisfies the conditions of Definition~\ref{def:mod-killing}. Decomposing a matrix $X = X_0 + \frac{1}{n}\mathrm{tr}(X)I$ (where $X_0 \in \mathfrak{sl}(n)$) reveals that our form separates orthogonally:
\begin{equation}
\widetilde{B}(X, Y) = \underbrace{2n \cdot \mathrm{tr}(X_0 Y_0)}_{B_{\mathfrak{sl}(n)}(X_0, Y_0)} + \underbrace{\mathrm{tr}(X)\mathrm{tr}(Y)}_{\text{Inner product on } \mathbb{R}I}.
\end{equation}
This decomposition directly shows how $\widetilde{B}$ serves as a generalization of prior work. The first term is the Killing form on the semisimple part, which is the tool used in Lie Neurons~\citep{lin2023lie}. The second term is a standard inner product on the center, which, under the isomorphism $\mathfrak{so}(3)\simeq\mathbb{R}^3$, recovers the dot product used in Vector Neurons~\citep{deng2021vector}. For $\mathfrak{so}(3)$, the hat map intertwines the vector action and the adjoint action, so $\mathrm{Ad}$-equivariance on $\mathfrak{so}(3)$ corresponds to standard $\mathrm{SO}(3)$-equivariance on $\mathbb{R}^3$; see Lemma~\ref{lem:so3_hat_adjoint}. Our single form thus unifies these approaches, extending to the full reductive algebra $\mathfrak{gl}(n)$ (details can be found in Appendix~\ref{app:connections}).

\begin{figure}[t]
    \centering
    \includegraphics[width=0.6\textwidth]{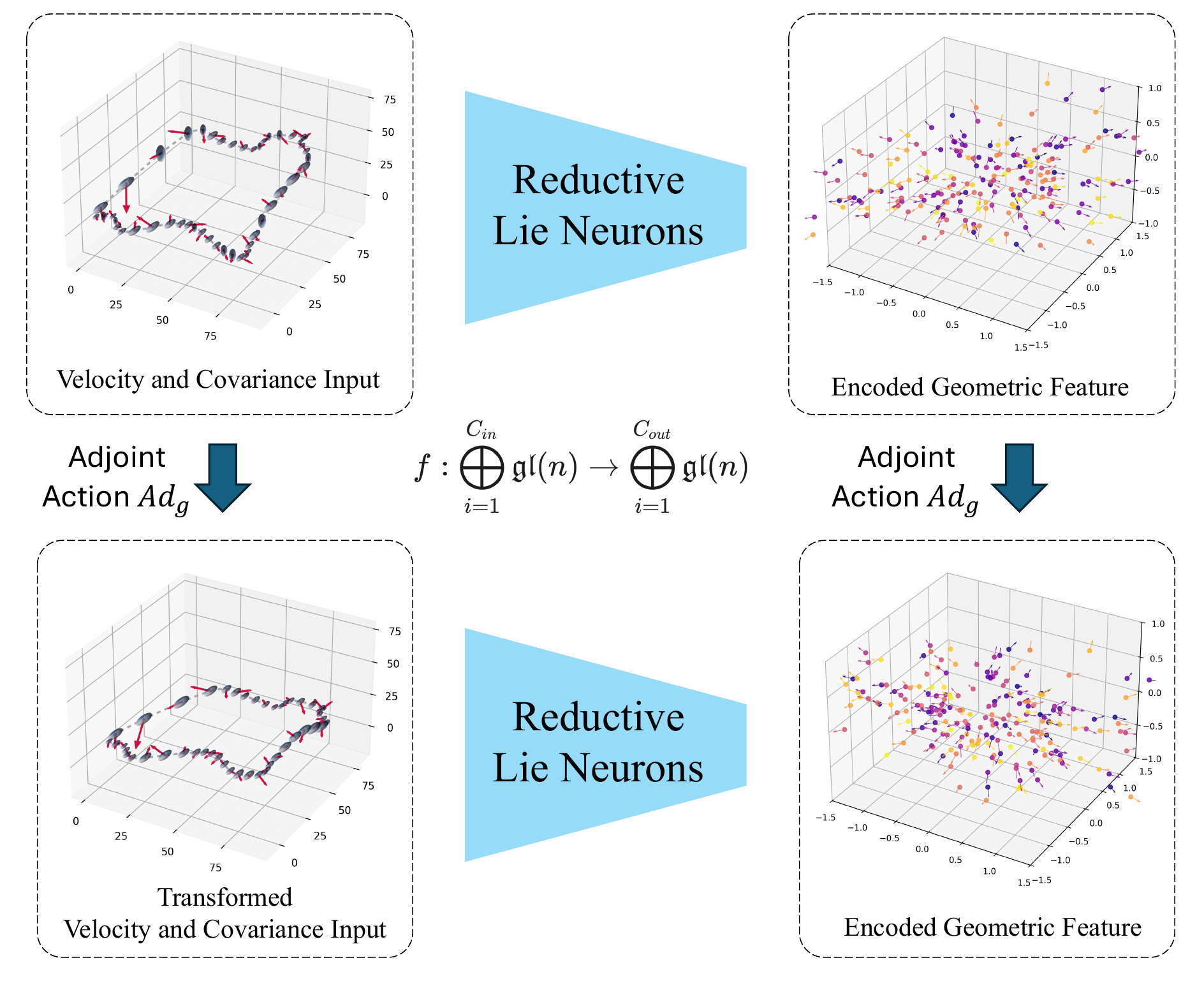}
    \caption{Adjoint equivariance using a unified representation for diverse geometric inputs. Our framework embeds inputs with different transformation rules, such as velocity (\(v \mapsto Rv\)) and covariance (\(\Sigma \mapsto R\Sigma R^T\)), into a common Lie algebra. Therefore, they transform under the same adjoint action \(\mathrm{Ad}_g\), with which our network \(f\) commutes as shown in the diagram.}
    \label{fig:equiv_framework}
\end{figure}

\subsection{The ReLN Layer Toolbox}

We represent multi-channel input as $x\in\mathbb R^{K\times C}$, where each column $x_c\in\mathbb R^K$ corresponds to a matrix $X_c\in\mathfrak g$ (via the vee/hat isomorphism, Appendix~\ref{app:prelim}) .

\paragraph{ReLN-Linear.}  A linear map applied to the channel dimension $f(x;W)=xW$ with $W\in\mathbb R^{C\times C'}$ is equivariant: the group acts on the left (geometric dimension) while $W$ acts on the right (channel dimension), and thus these operations commute (formal proof in Appendix~\ref{app:layer_details}).

\paragraph{Equivariant Nonlinearities.} We introduce two complementary nonlinear primitives; full definitions, parameterizations, and stability prescriptions are deferred to Appendix~\ref{app:layer_details}.
\begin{itemize}
  %   \item \textbf{ReLN-ReLU}: To overcome the non-equivariance of elementwise activations, we use our form $\widetilde{B}$ to define a directional nonlinearity (invariant gate). Each vector feature $x_c$ is rectified along a learnable direction $d_c$ via the update rule:
  %   % \textbf{ReLN-ReLU} Using $\widetilde B$ to build an invariant gate, each channel feature is updated by
  %   \begin{equation}
  %   x_c' \;=\; x_c \;+\; \max\big(0,\;\widetilde B(x_c^\wedge,d_c^\wedge)\big)\, d_c,
  %   \end{equation}
  % where $x_c^\wedge,d_c^\wedge\in\mathfrak g$ are the matrix forms of the channel feature and a learnable direction. Because $\widetilde B(\cdot,\cdot)$ is $\mathrm{Ad}$-invariant, the scalar gate is invariant and the vector update is equivariant.

    % \item \textbf{ReLN-ReLU}: To avoid basis-dependent elementwise activations, we use $\widetilde{B}$ to build an \emph{invariant scalar gate} and apply it along an \emph{equivariant direction}.
    % For each channel, we obtain a direction $d_c$ from $x$ via an equivariant map (e.g., an equivariant linear mixing), and update
    % \begin{equation}
    % x_c' \;=\; x_c \;+\; \sigma\!\big(\widetilde B(x_c^\wedge,d_c^\wedge)\big)\, d_c,
    % \qquad \sigma(t)=\max(0,t),
    % \end{equation}
    % where $x_c^\wedge,d_c^\wedge\in\mathfrak g$ denote the matrix forms.
    % Since $\widetilde B(\cdot,\cdot)$ is $\mathrm{Ad}$-invariant, the gate $\sigma(\widetilde B(\cdot,\cdot))$ is invariant; because $d_c$ transforms equivariantly with $x_c$, the resulting vector update is equivariant.

    \item \textbf{ReLN-ReLU}: To avoid basis-dependent elementwise activations, we use $\widetilde{B}$ to build an \emph{invariant scalar gate} and apply it along an \emph{equivariant direction}. For input features $x \in \mathbb{R}^{K \times C}$, we explicitly define the equivariant reference direction $d = xW$, where $W \in \mathbb{R}^{C \times C}$ is a learnable weight matrix. The feature is then updated via the following gated mechanism:
    \begin{equation}
        x_c' = \begin{cases} x_c, & \text{if } \widetilde B(x_c^\wedge,d_c^\wedge) \le 0 \\ x_c + \widetilde B(x_c^\wedge,d_c^\wedge) d_c, & \text{otherwise} \end{cases}
    \end{equation}
    where $x_c^\wedge, d_c^\wedge \in \mathfrak{g}$ denote the matrix forms. Since $\widetilde B(\cdot,\cdot)$ is $\mathrm{Ad}$-invariant, the gating condition is invariant; because $d$ is an equivariant linear map of $x$, the resulting vector update remains strictly equivariant.

    \item \textbf{ReLN-Bracket.} Following prior work~\citep{lin2023lie}, we include a layer that leverages the Lie bracket (matrix commutator). This operation is an $\mathrm{Ad}$-equivariant primitive on the Lie algebra that creates nonlinear interactions by measuring the non-commutativity of features. The layer applies two independent linear maps, parameterized by weights $W_a, W_b \in \mathbb{R}^{C \times C}$, to the input channels $x_{\mathrm{in}}$ to produce two intermediate features, computes their commutator, and injects the vectorized result as a shared residual:
    \begin{equation}
        x_{\mathrm{out}} = x_{\mathrm{in}} + \big(\left[ (x_{\mathrm{in}} W_a)^\wedge, (x_{\mathrm{in}} W_b)^\wedge \right]\big)^\vee.
    \end{equation}
    
\end{itemize}
\paragraph{Equivariant Pooling and Invariant Layers.}
The final components of the ReLN toolbox enable feature aggregation and the production of invariant outputs.

\begin{itemize}
    \item \textbf{Max-Killing Pooling:} Given an unordered set $\{x_n\}_{n=1}^N$ with $x_n\in\mathbb{R}^{K\times C}$, we form a \emph{dynamic} query via the shared ReLN-Linear map $d_n = x_n W_q$. For each channel $c$, we select the index $n^*(c)=\arg\max_n \widetilde{B}\big(X_{n,c}, (d_n)_c^\wedge\big)$ and return the pooled feature $X_c^{\max}=X_{n^*(c),c}$. Because $W_q$ acts only on channels, $(d_n)_c^\wedge$ transforms equivariantly with $X_{n,c}$, making the score $\widetilde{B}(\cdot,\cdot)$ invariant; parameter sharing across $n$ and the symmetric $\arg\max$ aggregation preserve permutation invariance.
    \item \textbf{Invariant Layer:} To produce a group-invariant output, this layer contracts feature $X_c$ using the form $\widetilde{B}$. The resulting scalar, $y_c = \widetilde{B}(X_c, X_c)$, is invariant by construction.
\end{itemize}

\paragraph{Unifying geometric representations.}
By operating directly on $n\times n$ matrix representations, ReLNs provide a unified primitive for vectors, matrices, and higher-order geometric objects (e.g., covariances). This allows ReLNs to handle a broader class of geometric inputs without resorting to separate specialized architectures; empirical validation is presented in Section~\ref{subsec:covariance}.

\section{Experiments}
\label{sec:experiments}

We evaluate ReLNs along three axes: (i) \emph{exact adjoint equivariance} under broad linear symmetries (including reductive and non-compact groups), (ii) \emph{learning from coupled heterogeneous geometry} (e.g., velocities with covariances), and (iii) \emph{accuracy--efficiency trade-offs} against widely used equivariant and self-supervised baselines.

We consider (a) algebraic benchmarks for adjoint-equivariant learning on $\mathfrak{sl}(3)$ and $\mathfrak{sp}(4)$, (b) uncertainty-aware drone state estimation with joint velocity--covariance processing under $\mathrm{SO}(3)$ frame changes, (c) the EMLP double-pendulum benchmark assessing accuracy and per-step efficiency, and (d) learning from 3D Gaussian-splat primitives in ShapeSplat~\cite{ma2025large} with comparisons to Gaussian-MAE. 
We compare against non-equivariant baselines (MLP, ResNet), Lie-algebraic equivariant models (Lie Neurons), and strong equivariant architectures including Vector Neurons (VN), Tensor Field Networks (TFN), and an $\mathrm{SE}(3)$-Transformer. Additional Lorentz-equivariant top-tagging results are deferred to Appendix~\ref{app:top_tagging}.

\subsection{Algebraic Benchmarks on Semisimple Lie Algebras}
To verify that our general $\mathfrak{gl}(n)$ framework correctly generalizes to semisimple subalgebras, we evaluate ReLN on two Lie-algebraic benchmarks first introduced by Lie Neurons~\citep{lin2023lie}.
% \tyl{Just out of curiosity, we didn't test anything related to $GL(N)$ right? We have been claiming about the equivariance to $GL(N)$ without experiments, probably expect the reveiwers will ask about it.}

\begin{table}[t]
  \centering
  \scriptsize
  \caption{Platonic solid classification (mean $\pm$ std over 5 runs). ID = in-distribution; RC = rotated-camera (10 random $\mathrm{SO}(3)$ test rotations). ``+Aug'' denotes training with random $\mathrm{SO}(3)$ augmentation of input homographies. Higher is better ($\uparrow$).}
  \label{tab:platonic}
  
  % [수정 1] 여백을 10pt -> 3pt로 줄임
  \setlength{\tabcolsep}{3pt}
  
  % [수정 2] 표 크기를 컬럼 너비에 딱 맞게 자동 조절
  \resizebox{0.8\linewidth}{!}{
      \begin{tabular}{l r c c}
        \toprule
        \textbf{Model} & \textbf{\# Params} & \textbf{ID Acc} (mean $\pm$ std) & \textbf{RC Acc} (mean $\pm$ std) \\
        \midrule
        MLP  & 206,339 & 95.76\% $\pm$ 0.65\% & 36.54\% $\pm$ 0.99\% \\
        MLP + Aug & 206,339 & 81.47\% $\pm$ 0.77\% & 81.20\% $\pm$ 2.34\% \\
        MLP (wider) & 411,479 & 96.82\% $\pm$ 0.53\% & 36.55\% $\pm$ 0.34\% \\
        MLP (wider) + Aug & 411,479 & 85.22\% $\pm$ 1.46\% & 83.43\% $\pm$ 0.51\% \\
        \midrule
        Lie Neurons& 331,272 & 99.62\% $\pm$ 0.25\% & 99.61\% $\pm$ 0.14\% \\
        \textbf{ReLN (Ours)} & 331,272 & \textbf{99.78\% $\pm$ 0.04\%} & \textbf{99.78\% $\pm$ 0.04\%} \\
        \bottomrule
      \end{tabular}
  }
\end{table}

\subsubsection{Platonic Solid Classification}
\label{subsec:platonic}

% \paragraph{Platonic Solid Classification on $\mathfrak{sl}(3)$.}

We first validate our model on the Platonic solid classification benchmark from~\citep{lin2023lie}, testing the adjoint-equivariance where camera rotations induce a conjugation action on the homographies between the projected faces of the solid, represented as $\mathrm{SL}(3)$. Full experimental details are provided in Appendix~\ref{app:hyperparams}.

The results, summarized in Table~\ref{tab:platonic}, confirm that non-equivariant baselines fail to generalize to rotated camera views. This generalization gap persists even with augmentation or increased capacity, as our wider MLP variant with approximately double the parameters shows negligible improvement on the out-of-distribution test set. In contrast, the ReLN model achieves near-perfect accuracy with robustness, matching the performance of the Lie Neurons while demonstrating improved results. This result validates that our general $\mathfrak{gl}(n)$ framework operates effectively on common semisimple subalgebras, as its built-in adjoint-equivariance on the parent group yields robust behavior when restricted to subgroups like $\mathrm{SO}(3)$ and $\mathrm{SL}(3)$.

\subsubsection{Invariant Function Regression on $\mathfrak{sp}(4)$.}
\label{subsec:sp4}

To probe algebraic generality beyond $\mathfrak{sl}(n)$, we regress a highly nonlinear $\mathrm{Sp}(4)$-invariant scalar on the real symplectic Lie algebra $\mathfrak{sp}(4,\mathbb{R})$ (the Lie algebra underlying Hamiltonian/symplectic symmetry). For $X,Y\in\mathfrak{sp}(4,\mathbb{R})$, the target is
\begin{align}
\nonumber g(X,Y)=&\sin\!\big(\tr(XY)\big)+\cos\!\big(\tr(YY)\big) \\&-\tfrac{1}{2}\tr(YY)^3
+\det(XY)+\exp\!\big(\tr(XX)\big).
\end{align} 
We sample 10k training and 10k test pairs and compare ReLN to MLP baselines (with/without $\mathrm{Sp}(4)$ conjugation augmentation) and Lie Neurons.
We report test MSE, MSE averaged over 500 random adjoint actions, and the invariance error.

As shown in Table~\ref{tab:sp4}, non-equivariant MLPs are orders of magnitude less accurate and exhibit substantially larger invariance error, indicating that they fail to reliably capture the underlying group structure even with \(\mathrm{Sp}(4)\) conjugation augmentation. In contrast, Lie Neurons and ReLN achieve low test MSE and near-zero invariance error, demonstrating stable invariant regression on \(\mathfrak{sp}(4)\). At this parameter scale, the performance of the two equivariant models is comparable. These results support that our model matches the accuracy of specialized Lie-algebraic baselines while retaining a unified construction across diverse Lie algebras.

\begin{table}[t] % * 제거, [t] 또는 [h] 사용
  \centering
  \scriptsize
  \caption{Regression performance and invariance error on \(\mathfrak{sp}(4)\). “Tr. Aug.” indicates training augmentation.}
  \label{tab:sp4}
  
  % 컬럼 간 여백을 줄여 공간 확보
  \setlength{\tabcolsep}{3pt} 
  
  % 표 전체 폭을 현재 컬럼 너비에 맞춤
  \resizebox{0.8\columnwidth}{!}{ 
    \begin{tabular}{l c r cc c}
      \toprule
      \multirow{2}{*}{Model}
        & Tr. Aug. % 헤더 길이를 살짝 줄임
        & \multirow{2}{*}{\# Params}
        & \multicolumn{2}{c}{Test Aug.}
        & Inv. Error \\ 
      \cmidrule(lr){4-5}
        & & & ID & SP(4) & \\
      \midrule
      \multirow{2}{*}{MLP 256}  
        & Id    & 137{,}217 & 0.126 & 1.360 & 0.722 \\ 
        & SP(4) & 137{,}217 & 0.192 & 0.587 & 0.476 \\
      \multirow{2}{*}{MLP 512}  
        & Id    & 536{,}577 & 0.107 & 0.906 & 0.585 \\ 
        & SP(4) & 536{,}577 & 0.123 & 0.446 & 0.374 \\
      \midrule
      Lie Neurons  
        & Id    & 263{,}170 & \(5.83\times10^{-4}\) & \(5.84\times10^{-4}\) & \({3.84\times10^{-7}}\) \\
      ReLN (ours)
        & Id    & 263{,}170 & \({5.14\times10^{-4}}\) & \({5.14\times10^{-4}}\) & \(4.73\times10^{-7}\) \\
      \bottomrule
    \end{tabular}
  }
\end{table}

\subsection{Drone State Estimation with Geometric Uncertainty}
\label{subsec:covariance}

We evaluate ReLN on a drone state-estimation task that reconstructs 3D trajectories from noisy velocity measurements $\mathbf{v}\in\mathbb{R}^3$ and time-varying uncertainty covariances $C\in\mathrm{SPD}(3)$ during highly dynamic flights. This setup requires a model to jointly process vector and matrix data in a geometrically consistent and uncertainty-aware manner.

\paragraph{Experimental Setup.}
We use a large-scale synthetic dataset of 200 aggressive trajectories (over 13 hours of flight); details are in Appendix~\ref{app:drone_exp_details}.
We evaluate ReLN against: non-equivariant 1D ResNets, VN~\citep{deng2021vector}, and two spherical-harmonics-based equivariant architectures that use steerable bases and tensor-product coupling to construct equivariant messages: TFN~\citep{thomas2018tensor} and SE(3)-Transformers~\citep{fuchs2020se}.

To isolate the effect of uncertainty fusion, we test three variants for each equivariant model: velocity-only, velocity + covariance, and velocity + log-covariance. For baselines such as VN that cannot natively process matrix inputs, we implement an eigendecomposition-based strategy to ingest covariance data (details in Appendix~\ref{app:impl}). To ingest covariance data in TFN, SE(3)-Transformer, we apply irreducible representation (irrep) decomposition. 

\paragraph{Irrep decomposition vs.\ unified adjoint processing.}

Steerable baselines (TFN, SE(3)-Transformer) represent $C$ through irreducible components (trace and traceless-symmetric parts, corresponding to $\ell=0$ and $\ell=2$), and learn interactions with $\mathbf{v}$ via tensor-product coupling between fibers. This decomposition ignores the holistic geometry of the measurement-uncertainty pair, relying on tensor-product coupling to reconnect these components. See Appendix~\ref{app:irrep_decomp} for the detailed architecture design and decomposition. In contrast, ReLN embeds both $\mathbf{v}$ and $C$ into a common matrix-valued representation with a single conjugation rule: $\mathbf{v}$ is lifted to $\mathfrak{so}(3)\subset\mathfrak{gl}(3)$ and $C$ (or $\log C$) is provided as a matrix feature.
This yields a uniform $\mathrm{Ad}$-action for the pair and allows gating/normalization using the same $\mathrm{Ad}$-invariant non-degenerate form $\widetilde B$. The final velocity estimate is obtained by projecting the network output to its skew-symmetric component; implementation details are deferred to Appendix~\ref{app:drone_exp_details}.

\subsubsection{Results and Analysis}
\label{sec:results_analysis}

\begin{table}[t]
\centering
\caption{Drone state estimation performance (ATE and RTE in meters). $v, C, \log C$ denote velocity, covariance, and log-covariance inputs. Refer to Appendix~\ref{app:drone_full_results} for the exhaustive performance report including all variants and metrics.}
\label{tab:drone_results_slim}
\scriptsize
\setlength{\tabcolsep}{2.5pt} % 컬럼 간격 최소화
\resizebox{0.9\columnwidth}{!}{
\begin{tabular}{l c SS SS}
\toprule
\multirow{2}{*}{\textbf{Model}} & \multirow{2}{*}{\textbf{Input}} & \multicolumn{2}{c}{\textbf{ID}} & \multicolumn{2}{c}{\textbf{SO(3)}} \\
\cmidrule(lr){3-4} \cmidrule(lr){5-6}
& & {ATE $\downarrow$} & {RTE $\downarrow$} & {ATE $\downarrow$} & {RTE $\downarrow$} \\
\midrule
\textit{Non-Equiv.} \\
\quad ResNet & $(v, C)$ & 205.11 & 106.07 & 213.26 & 109.37 \\
\midrule
\textit{Equiv. Baselines} \\
\quad VN \citep{deng2021vector} & $v$ & 17.36 & 13.51 & 17.36 & 13.51 \\
\quad VN \citep{deng2021vector} & $(v, C)$ & 191.78 & 98.39 & 190.22 & 98.26 \\
\quad TFN \citep{thomas2018tensor} & $(v, C)$ & 17.56 & 14.40 & 17.56 & 14.40 \\
\quad SE(3)-Tr. \citep{fuchs2020se} & $(v, C)$ & 21.67 & 16.77 & 21.67 & 16.77 \\
\quad Lie Neurons \citep{lin2023lie} & $(v, C)$ & 16.86 & 13.65 & 16.86 & 13.65 \\
\midrule
% \textit{Lie-algebraic} \\
\textit{Our Equivariant Models} \\
\quad {ReLN (Ours)} & $v$ & 16.85 & 12.70 & 16.85 & 12.70 \\
\quad {ReLN (Ours)} & $(v, C)$ & 16.49 & 13.02 & 16.49 & 13.02 \\
\quad {ReLN (Ours)} & $(v, \log C)$ & \textbf{13.92} & \textbf{11.04} & \textbf{13.92} & \textbf{11.04} \\
\quad {ReLN (Ours)} (no semisimple) & $(v, \log C)$ & 16.27 & 12.65 & 16.27 & 12.65 \\

\bottomrule
\end{tabular}
}
\end{table}

\begin{figure*}[htbp] % Use [htbp] for good placement (Here, Top, Bottom, Page)
    % \centering
    \begin{center}
    \includegraphics[width=\textwidth]{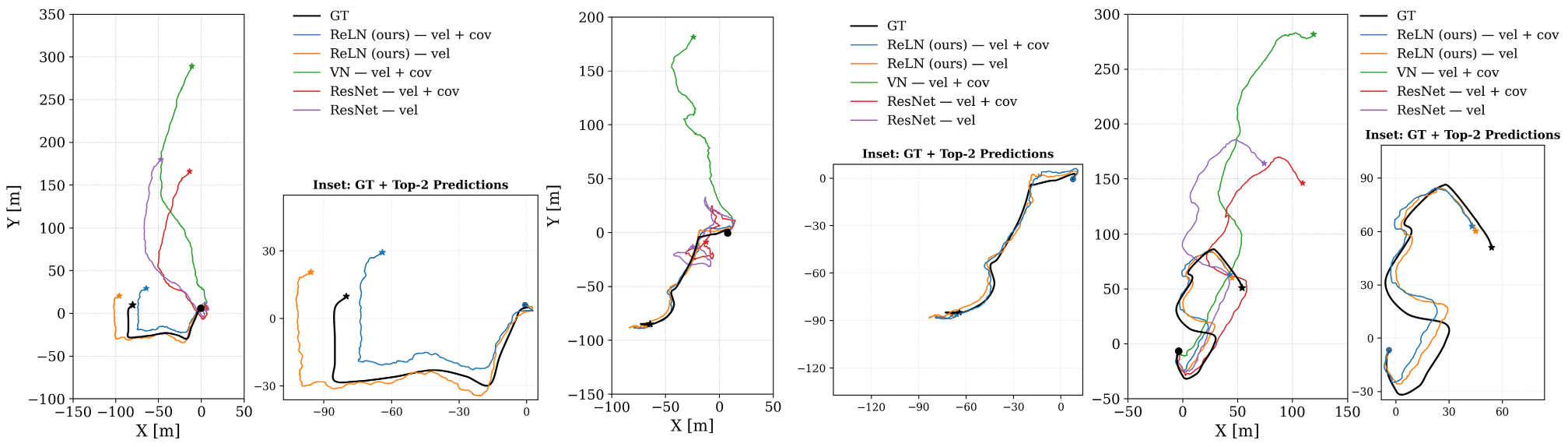}
    
    % The caption text you requested.
    \caption{
        Trajectory reconstruction across flight difficulties: high (left), best-case (middle), and medium (right). ReLN models consistently track the ground truth (black) with high fidelity, especially when leveraging covariance. Insets provide a magnified view of the two best-performing variants (ReLN) to highlight their accuracy.
    }
    % A label is crucial for referencing the figure in your text, e.g., "As shown in Figure~\ref{fig:main_trajectories}..."
    \label{fig:main_trajectories}
    \end{center}
\end{figure*}

Table~\ref{tab:drone_results_slim} summarizes the main findings (full results in Appendix~\ref{app:drone_full_results}).

\textbf{Equivariance is necessary for robustness and accuracy.} Non-equivariant ResNets exhibit large trajectory errors and do not reliably benefit from adding covariance, indicating that simply concatenating matrix-valued uncertainty does not reliably yield geometry-consistent fusion. Moreover, their errors remain high under test-time $\mathrm{SO}(3)$ rotations, confirming poor robustness to measurement-frame changes.

% Second, \emph{being equivariant is not sufficient---the interface for covariance matters.}
% Among equivariant baselines, velocity-only variants already perform well (e.g., VN, TFN, and SE(3)-Transformer), but their behavior diverges once covariance is introduced.
% Notably, the VN covariance variant degrades to near non-equivariant performance, suggesting that an eigendecomposition-based interface can break the coupled structure between a measurement and its uncertainty (e.g., separating eigenvectors from eigenvalues), making it difficult to learn a stable uncertainty-aware weighting.

\textbf{The Necessity of Structural Compatibility in Covariance Fusion.} Among equivariant baselines, velocity-only models already attain low error (e.g., VN), but performance becomes highly sensitive once covariance is introduced.
In particular, the VN $(v,C)$ variant collapses toward non-equivariant performance (Table~\ref{tab:drone_results_slim}; Table~\ref{tab:drone_full_comparison}), consistent with the fact that an eigendecomposition-based interface can separate coupled degrees of freedom (eigenvectors vs.\ eigenvalues) and yield an unstable or non-geometric uncertainty signal for learning.  By contrast, TFN and SE(3)-Transformer can incorporate $(v,C)$ through steerable tensor representations, leading to moderate improvements.

\textbf{Local equivariant operators offer superior robustness for uncertainty fusion.}
While SE(3)-Transformer shows a slight performance advantage in the velocity-only regime ($v$), the introduction of covariance uncertainty shifts the advantage to TFN. As shown in Table~\ref{tab:drone_full_comparison}, TFN consistently improves over SE(3)-Transformer across all covariance-integrated variants (e.g., $(v, \log C)$). This pattern suggests that while attention mechanisms can effectively model global dependencies in simple velocity sequences, local convolutional aggregation provides a stronger and more stable inductive bias for processing high-frequency dynamics and uncertainty.

\paragraph{ReLN yields the most consistent uncertainty-aware gains.}
ReLN achieves the best overall performance, and its strongest configuration uses $(v,\log C)$ (ATE $13.92$, RTE $11.04$). Unlike tensorial pipelines, ReLN represents velocity and covariance in a shared matrix-algebraic space with a single conjugation rule, enabling uncertainty-conditioned contractions and gating without violating equivariance. The improvement from $C$ to $\log C$ further supports that respecting the intrinsic geometry of $\mathrm{SPD}(3)$ via the matrix logarithm improves conditioning for learning.

\paragraph{Ablative Validation of Reductive Decomposition.}
ReLN improves over its \emph{semisimple-only} variant (the no-center model, equivalent to Lie Neurons), indicating that incorporating the central component $\mathfrak z(\mathfrak g)$ is beneficial for uncertainty-aware estimation. Conversely, removing the semisimple ideal (retaining only the center) degrades performance, increasing ATE from $13.92$ to $16.27$ (Table~\ref{tab:drone_results_slim}). Together, these ablations support that the best performance arises from modeling the \emph{full} reductive structure.

\paragraph{Discussion.}
\label{subsubsec:discussion_cov}
Overall, the results indicate that ReLNs act as a \emph{geometry- and uncertainty-aware} estimator, remaining stable under random measurement-frame changes (test-time $\mathrm{SO}(3)$ rotations). While ReLN is not a classical recursive (Markovian) filter, it learns to fuse velocity sequences with their time-varying covariances through uncertainty-dependent modulation, improving trajectory reconstruction. This suggests ReLNs as a modular backbone for estimation pipelines that must handle matrix-valued uncertainty. A natural direction is to characterize when the learned fusion approximates uncertainty-adaptive integration, and to integrate ReLNs with classical state-estimation frameworks.

\subsection{Equivariance for 3D Gaussian Splatting (3DGS) }
\label{sec:3dgs_experiment}

\begin{figure*}[h]
    \centering
    \includegraphics[width=1.0\linewidth]{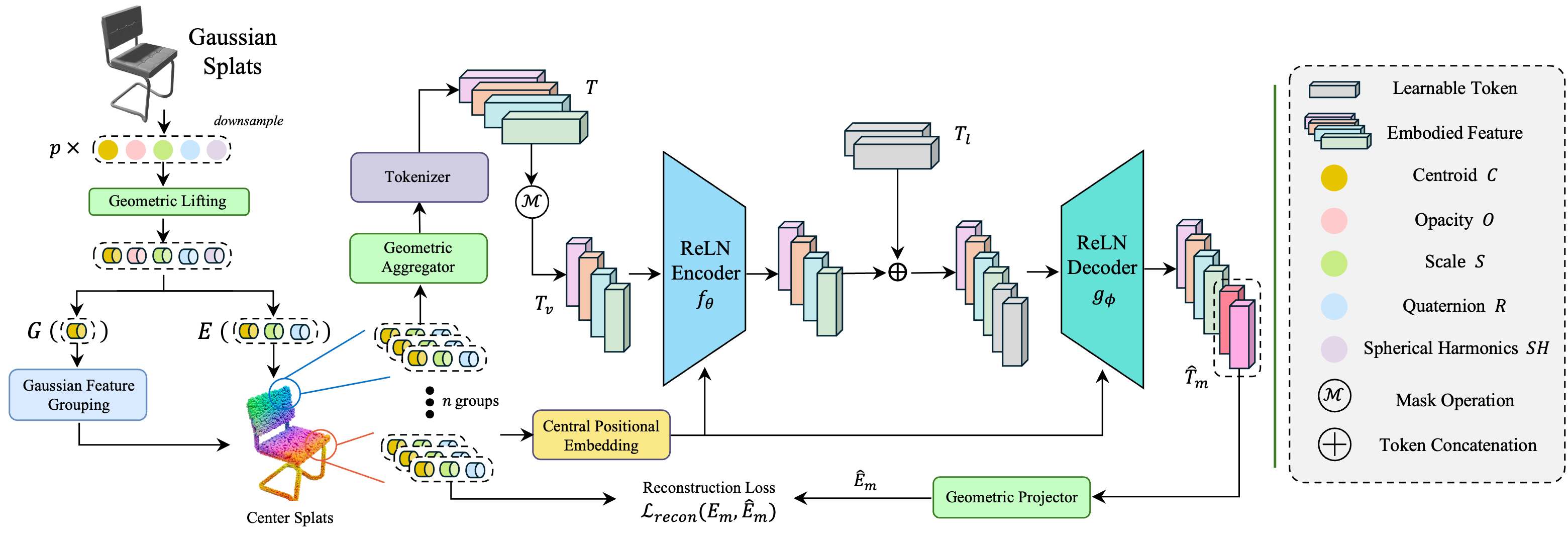} 
    \caption{\textbf{The ReLN-integrated Gaussian-MAE Framework.} Our architecture pipelines raw 3D Gaussian splats through a \textbf{Geometric Lifting} stage to map them into the $\mathfrak{gl}(3)$ Lie algebra. The model explicitly bifurcates the data flow: (1) \textbf{Active Geometry} (position $\mu$, covariance $\Sigma$) is processed via the \textbf{ReLN Encoder} and \textbf{Decoder} to maintain $\mathrm{GL}(3)$-equivariance, while (2) \textbf{Invariant Attributes} (color $c$, opacity $\alpha$) are integrated as Type-0 features. A dedicated \textbf{Central Positional Embedding} ensures stable spatial context. Finally, the \textbf{Geometric Projector} utilizes the Vee map and Killing form to reconstruct the physical Gaussian parameters while preserving holistic symmetry.}
    \label{fig:framework_overview}
\end{figure*}

Standard point cloud networks typically process only spatial coordinates. However, 3DGS represents scenes using anisotropic 3D Gaussians parameterized by a mean position $\mu \in \mathbb{R}^3$ and a covariance matrix $\Sigma \in \mathrm{SPD}(3)$. This introduces a geometric challenge: a global rotation $R \in \mathrm{SO}(3)$ acts differently on these components, where $\mu$ transforms as a vector ($R\mu$) and $\Sigma$ transforms via congruence ($R\Sigma R^\top$).

To evaluate ReLNs in this multi-modal geometric setting, we adapt the masked autoencoder framework proposed by \citet{ma2025large}, which establishes a self-supervised pretraining baseline for 3D Gaussian Splats. While their original architecture processes Gaussian parameters as loosely coupled features using standard Transformers, we redesign the encoder and decoder using ReLN blocks to enforce $\mathrm{GL}(3)$-equivariance. The overall architecture of our ReLN-integrated Gaussian-MAE is illustrated in Figure~\ref{fig:framework_overview}. Specifically, we unify these geometric types by embedding them into the Lie algebra $\mathfrak{gl}(3)$: the mean position $\mu$ is treated as a translation, while the covariance $\Sigma$ is mapped to the linear subspace $\mathrm{Sym}(3) \subset \mathfrak{gl}(3)$ via the matrix logarithm. This ensures that the network learns consistent geometric representations of 3D shapes under arbitrary orientations.

\textbf{Experimental Setup.} We utilize the ShapeNet dataset processed into 3D Gaussian primitives for self-supervised pre-training, following the protocol of \citet{ma2025large}. The model is then fine-tuned on ModelNet10 for classification. We compare our ReLN-based architecture against the baseline model from \citet{ma2025large}. Refer to Appendix~\ref{app:3dgs_details} for full experimental details.

\paragraph{Experimental Results.} Table~\ref{tab:rotation_robustness} evaluates classification accuracy on aligned and randomly rotated ($0^\circ$--$180^\circ$) test sets. The baseline drops from $93.39\%$ to $18.28\%$. In contrast, ReLN maintains a stable accuracy of $95.15\%$, representing a marginal $0.33$ point variance compared to its aligned performance. Furthermore, ReLN converges in fewer epochs during pretraining and achieves lower reconstruction error across attributes (including rotation and scale); see Appendix Figure~\ref{fig:loss_convergence}. Together with the rotation-stable accuracy in Table~\ref{tab:rotation_robustness}, these results support that enforcing the coupled transformation rules of $(\mu,\Sigma)$ via a reductive Lie-algebraic backbone improves robustness to orientation changes.

% --- TABLE FOR 3DGS RESULTS ---
\begin{table}[t]
    \centering
    \scriptsize
    \caption{\textbf{Rotation robustness on 3D Gaussian splats.} 
    Classification accuracy (\%) on ModelNet10 for the baseline \citep{ma2025large} and ReLN. 
    \textit{Standard} uses the aligned test set; \textit{Rotated} applies random rotations ($0^\circ$--$180^\circ$).}
    \label{tab:rotation_robustness}
    \resizebox{0.7\columnwidth}{!}{
        \begin{tabular}{lcc}
            \toprule
            \textbf{Method} & \textbf{Standard Acc.} (\%) & \textbf{Rotated Acc.} (\%) \\
            \midrule
            Baseline \citep{ma2025large} & 93.39 & 18.28 \\
            \textbf{ReLN (Ours)}         & \textbf{94.82} & \textbf{95.15} \\
            \bottomrule
        \end{tabular}
    }
\end{table}

%%%%%%%%%%%%%%%%%%%% EMLP Benchmark  %%%%%%%%%%%%%%%%%%%%

\subsection{EMLP Double-Pendulum Benchmark and Efficiency}
We evaluate ReLN on the EMLP double-pendulum benchmark~\cite{finzi2021practical}, which learns Hamiltonian dynamics under multiple symmetry groups ($O(2)$, $SO(2)$, and $D_6$). Following the standard protocol, we report test rollout errors in Table~\ref{tab:double_pendulum}; additional details are provided in Appendix~\ref{app:double_pendulum_details}.

\textbf{Accuracy.} ReLN achieves comparable or lower rollout error than EMLP across all symmetry settings (e.g., $SO(2)$: \textbf{0.010} vs.\ $0.015$), without introducing group-specific architectural modifications.

\textbf{Compute and latency.} With matched hidden size and representation dimension, ReLN-HNN is substantially cheaper per step than EMLP-HNN (Table~\ref{tab:efficiency}): $142{,}190$ vs.\ $1{,}589{,}909$ FLOPs ($11.18\times$) and $2.216$ vs.\ $61.635$ ms ($27.81\times$). This gap is consistent with ReLN using closed-form matrix operations from exact $\mathrm{Ad}$-equivariant primitives on $\mathfrak{gl}(n)$, whereas EMLP relies on symmetry-dependent basis construction and projection.

Overall, these results support ReLN as a reusable equivariant backbone that attains EMLP-level accuracy while substantially reducing per-step computation across different symmetry choices.

\begin{table}[t]
    \centering
    \caption{Test rollout error on the EMLP Double-Pendulum Benchmark. ReLN matches or improves upon EMLP across different symmetry groups ($O(2)$, $SO(2)$, $D_6$). Values represent the geometric mean of relative errors, with the standard deviation over 3 trials in parentheses; e.g., the notation $0.012 (2)$ denotes $0.012 \pm 0.002$.}
    \label{tab:double_pendulum}
    \resizebox{\linewidth}{!}{
    \begin{tabular}{l c c c c c c c}
        \toprule
        \multirow{2}{*}{\textbf{Metric}} & \multicolumn{2}{c}{\textbf{$O(2)$}} & \multicolumn{2}{c}{\textbf{$SO(2)$}} & \multicolumn{2}{c}{\textbf{$D_6$}} & \multirow{2}{*}{\textbf{MLP-HNNs}} \\
        \cmidrule(lr){2-3} \cmidrule(lr){4-5} \cmidrule(lr){6-7}
         & EMLP & \textbf{ReLN} & EMLP & \textbf{ReLN} & EMLP & \textbf{ReLN} & \\
        \midrule
        Rollout Error & 0.012 (2) & \textbf{0.011 (2)} & 0.015 (3) & \textbf{0.010 (4)} & 0.013 (2) & \textbf{0.011 (2)} & 0.028 \\
        \bottomrule
    \end{tabular}
    }
\end{table}

\begin{table}[t]
    \centering
    \caption{Computational cost on the HNN task. ReLN-HNN reduces FLOPs/step and inference latency relative to EMLP-HNN.}
    % \caption{Computational Efficiency (FLOPs) comparison on the HNN Task. ReLN is approximately $11\times$ more efficient than EMLP.}
    \label{tab:efficiency}
    \scriptsize % 기본 폰트 크기 조절
    \setlength{\tabcolsep}{4pt} % 컬럼 간 간격 축소 (기본값은 보통 6pt)
    \resizebox{0.7\columnwidth}{!}{ % 컬럼 너비에 맞춰 강제 조정
    \begin{tabular}{l c c c}
        \toprule
        \textbf{Model} & \textbf{\# Params} & \textbf{FLOPs / step} & \textbf{Inference (ms)} \\
        \midrule
        MLP-HNN & 34,817 & 70,400 & 0.159 \\
        \hdashline
        EMLP-HNN & 55,569 & 1,589,909 & 61.635 \\
        \textbf{ReLN-HNN (Ours)} & 69,889 & \textbf{142,190} & \textbf{2.216} \\
        \bottomrule
    \end{tabular}
    }
\end{table}
%%%%%%%%%%%%%%%%%%%% EMLP Benchmark  %%%%%%%%%%%%%%%%%%%%

\section{Discussion and Limitations.}
ReLNs are exactly $\mathrm{GL}(n)$-adjoint equivariant by construction. For covariance-like quantities, however, the physically standard transformation is congruence, $\Sigma \mapsto A\,\Sigma\,A^\top$, which agrees with adjoint conjugation $X\mapsto AXA^{-1}$
only for orthogonal changes of frame ($A^{-1}=A^\top$). Thus, while our unified Lie-algebraic interface is most directly interpretable under rigid rotations,
extending the same measurement--uncertainty coupling to general $\mathrm{GL}(n)$ coordinate changes requires additional geometric justification and may call for alternative liftings or invariants. A further direction is to extend the framework to groups with translations, such as affine and $\mathrm{SE}(3)$ symmetries, by leveraging their semidirect-product structure and the induced actions on the corresponding Lie algebras.

\section{Conclusion}
% This work introduces ReLNs, a unified neural architecture that provides exact equivariance to the adjoint action of the general $n \times n$ matrix algebra $\mathfrak{gl}(n)$ and its subgroups. ReLNs enable efficient learning on Lie-algebraic features and structured geometric data, such as covariance matrices. 
% Furthermore, our work establishes a unified Lie-algebraic framework that handles both classical left-action symmetries on vectors and native adjoint-actions on matrices within a single architecture. ReLNs achieve state-of-the-art results on benchmarks and deliver large gains in a challenging drone state estimation task by integrating uncertainty. We will apply our equivariant matrix processing capability to a wider array of physical systems, including the dynamics of articulated robots and large-scale sensor fusion, to further expand the boundaries of geometric deep learning.

We introduced Reductive Lie Neurons (ReLNs), a unified architecture that is exactly equivariant to the adjoint action on $\mathfrak{gl}(n)$ and its reductive matrix subalgebras. Using a non-degenerate $\mathrm{Ad}$-invariant contraction, ReLNs provide an efficient Lie-equivariant framework for learning with structured geometric quantities, including matrix-valued uncertainty. The same viewpoint links classical left actions on vectors to adjoint actions on matrices via explicit Lie-algebraic liftings, enabling a single backbone to process heterogeneous geometric inputs in a consistent manner. Empirically, ReLNs achieve strong performance on algebraic benchmarks and 3D Gaussian-splat representation learning, and yield substantial gains on drone state estimation by incorporating uncertainty equivariantly.

We expect these capabilities to benefit data-efficient representation learning in vision and robotics, particularly for structured 3D scene models—such as 3D Gaussian splats and world models—where coupled transformation rules are central. Future work will scale ReLNs to larger physical systems (e.g., multimodal fusion and articulated dynamics) and further explore efficient equivariant architectures for matrix-parameterized representations.

\subsubsection*{Acknowledgments}
% Use unnumbered third level headings for the acknowledgments. All
% acknowledgments, including those to funding agencies, go at the end of the paper.
This work was supported by AFOSR MURI FA9550-23-1-0400 and AFOSR YIP FA9550-25-1-0224. The authors thank Dr. Frederick A. Leve for
his encouragement and support.
%%%%%%%%%%%%%%%%%%%%

\clearpage
\bibliography{strings-full,ieee-full,references}
\bibliographystyle{iclr2026_conference}

\clearpage
\appendix

\section{Lie-theoretic preliminaries}\label{app:prelim}

This appendix provides an overview of key concepts and derivations from Lie group theory relevant to our construction of $\mathrm{GL}(n)$ adjoint-equivariant neural networks.

\subsection{Lie groups, Lie algebras, hat/vee}
\label{app:adjoint}
A \emph{matrix Lie group} $G\subset\mathrm{GL}(n)$ is a smooth subgroup of invertible matrices. Its Lie algebra $\mathfrak{g}=\operatorname{Lie}(G)$ is the tangent space at the identity and is identified with a subspace of $\mathfrak{gl}(n)$. Fix a basis $\{E_i\}_{i=1}^m$ of $\mathfrak{g}$. The coordinate maps are:
\begin{equation}
\wedge:\mathbb{R}^m\to\mathfrak{g},\quad x=(x_i)\mapsto x^\wedge=\sum_{i} x_i E_i,
\qquad
\vee:\mathfrak{g}\to\mathbb{R}^m,\quad X\mapsto X^\vee.
\end{equation}
These maps let us implement algebra-valued features as Euclidean vectors in code.

The associated Lie algebra \( \mathfrak{g} = \mathrm{Lie}(G) \) is the tangent space at the identity element \( e \in G \). It carries a bilinear, antisymmetric product called the \emph{Lie bracket}, given by
\begin{equation}
[A, B] = AB - BA,
\end{equation}
in the case of \( \mathrm{GL}(n) \) which captures the infinitesimal structure of the group near the identity. The bracket quantifies non-commutativity of generators: \( [A, B] = 0 \) implies commutativity, whereas \( [A, B] \neq 0 \) indicates a non-trivial interaction.

\subsection{Representations and the adjoint action}
\label{app:adjointrep}
A representation $\Phi:G\to\mathrm{GL}(V)$ differentiates to $\phi:\mathfrak{g}\to\mathfrak{gl}(V)$ by
\begin{equation}
\phi(X)=\left. \frac{d}{dt} \right|_{t=0}\Phi(\exp(tX)).
\end{equation}
The adjoint representation $\mathrm{Ad}:G\to\mathrm{GL}(\mathfrak{g})$ is defined to be the differential of group conjugation at the identity 
\begin{equation}
    \mathrm{Ad}_g(X) = \left. \frac{d}{dt}\right|_{t=0}g(\exp(tX))g^{-1}.
\end{equation}
Therefore we get a map $\mathrm{Ad}:G  \to \mathrm{GL}(\mathfrak g)$. For matrix groups, this is given 
\begin{equation}
\mathrm{Ad}_g(X)=gXg^{-1},\qquad g\in G,\; X\in\mathfrak{g}, \label{eq:adjoint-group}
\end{equation}
and differentiating yields the Lie-algebra adjoint $\mathrm{ad}_X(Y)=[X,Y]$. One checks
\begin{equation}
\mathrm{Ad}_g([X,Y])=[\mathrm{Ad}_g X,\mathrm{Ad}_g Y],\qquad \mathrm{ad}_X([Y,Z])=[\mathrm{ad}_X Y,Z]+[Y,\mathrm{ad}_X Z].
\end{equation}

\subsection{Vectorized adjoint}
Using the hat ($\wedge$) / vee ($\vee$) maps, the adjoint action on the Lie algebra induces a corresponding action on the vector coordinates. This vectorized action is a linear map represented by a matrix:
\begin{equation}
\mathrm{Ad}^m_g:\mathbb{R}^m\to\mathbb{R}^m,\qquad
\mathrm{Ad}^m_g(x) = (\mathrm{Ad}_g(x^\wedge))^\vee.
\end{equation}
Equivalently, $\mathrm{Ad}^m_g$ is the $m\times m$ matrix representation of $\mathrm{Ad}_g$ in the chosen basis $\{E_i\}$. In practice we precompute or assemble the $m\times m$ matrix representing $\mathrm{Ad}^m_g$ (or apply it implicitly) to implement left-multiplicative equivariant layers that act on vector features.

\subsection{Structure of Lie Algebra: Semisimplicity and Reductivity}
\label{app:reductive}
\begin{definition}[Semisimple and Reductive Lie Algebras]
A finite-dimensional Lie algebra $\mathfrak g$ is
\begin{itemize}
\item \emph{semisimple} if it has no nonzero solvable ideals (equivalently, its radical is zero);
\item \emph{reductive} if it decomposes as a direct sum of ideals
\begin{equation}
\mathfrak g \;=\; \mathfrak z(\mathfrak g)\;\oplus\;[\mathfrak g,\mathfrak g],
\end{equation}
where $\mathfrak z(\mathfrak g)$ is the center and $[\mathfrak g,\mathfrak g]$ is semisimple.
\end{itemize}
\end{definition}

\begin{example}
The Lie algebra \( \mathfrak{gl}(n) \) decomposes as:
\begin{equation}
\mathfrak{gl}(n)=\mathbb RI\oplus\mathfrak{sl}(n),
\qquad
\mathbb RI=\mathfrak z(\mathfrak{gl}(n)),\quad
\mathfrak{sl}(n)=[\mathfrak{gl}(n),\mathfrak{gl}(n)].
\end{equation}
where \( \mathfrak{sl}(n) \) (traceless matrices) is semisimple, and \( \mathbb{R}I \) (scalar matrices) forms the center. Thus $\mathfrak{gl}(n)$ is reductive but not semisimple (it has a nontrivial center).
\end{example}

This decomposition highlights the non-semisimple nature of \( \mathfrak{gl}(n) \), which plays a critical role in understanding the degeneracy of certain invariant forms such as the Killing form. This degeneracy hinders the application of standard tools in Lie-theoretic deep learning. Our work addresses this issue in the context of \( \mathrm{GL}(n) \)-equivariant architectures in Lie algebra $\mathfrak{gl}(n)$.

The Killing form on $\mathfrak g$ is
\begin{equation}
B_{\mathfrak g}(X,Y):=\operatorname{tr}(\mathrm{ad}_X\circ \mathrm{ad}_Y).
\end{equation}

\begin{theorem}[Cartan criterion]
\label{thm:cartan_criterion}
For a finite-dimensional Lie algebra $\mathfrak g$, the Killing form $B_{\mathfrak g}$ is non-degenerate if and only if $\mathfrak g$ is semisimple.
\end{theorem}

In particular, if $\mathfrak g$ has nontrivial center then $B_{\mathfrak g}$ is degenerate, since $\mathrm{ad}_Z=0$ for all $Z\in\mathfrak z(\mathfrak g)$.

%------------------------------------------------------------
\subsection{$\mathrm{Ad}$-invariant bilinear forms and the trace form}
\label{app:invariant_forms}

\begin{definition}[$\mathrm{Ad}$-invariance]
A bilinear form $B:\mathfrak g\times\mathfrak g\to\mathbb R$ is \emph{$\mathrm{Ad}$-invariant} if
\begin{equation}
B(\mathrm{Ad}_g X,\mathrm{Ad}_g Y)=B(X,Y)\qquad \forall g\in G,\;\forall X,Y\in\mathfrak g.
\end{equation}
\end{definition}

On semisimple $\mathfrak g$, the Killing form is $\mathrm{Ad}$-invariant and non-degenerate.
On reductive but non-semisimple algebras (notably $\mathfrak{gl}(n)$), one often uses alternative $\mathrm{Ad}$-invariant forms.
A basic example on $\mathfrak{gl}(n)$ is the trace pairing
\begin{equation}
\langle X,Y\rangle_{\mathrm{tr}}:=\operatorname{tr}(XY).
\end{equation}

\begin{proposition}[$\mathrm{Ad}$-invariance of the trace pairing on $\mathfrak{gl}(n)$]
\label{prop:trace_ad_invariant}
For $g\in \mathrm{GL}(n)$ and $X,Y\in\mathfrak{gl}(n)$,
\begin{equation}
\operatorname{tr}\bigl((gXg^{-1})(gYg^{-1})\bigr)=\operatorname{tr}(XY).
\end{equation}
\end{proposition}
\begin{proof}
By cyclicity of trace,
$\operatorname{tr}((gXg^{-1})(gYg^{-1}))=\operatorname{tr}(gXYg^{-1})=\operatorname{tr}(XY)$.
\end{proof}

The trace form is non-degenerate as a bilinear form on the vector space $\mathfrak{gl}(n)$ and therefore provides a practical substitute for the Killing form when designing $\mathrm{Ad}$-invariant bilinear layers on $\mathfrak{gl}(n)$.

We will use explicit $\mathfrak{gl}(n)$ identities (including the closed form of $B_{\mathfrak{gl}(n)}$, the resulting degeneracy, and our non-degenerate $\mathrm{Ad}$-invariant modification $\widetilde B$) in Appendix~\ref{app:connections}.

\section{Connections for $\mathfrak{gl}(n)$ and $\mathfrak{so}(3)\simeq\mathbb R^3$}
\label{app:connections}

This appendix collects two specialization facts used in the main text:
(i) an explicit expression for the (degenerate) Killing form on $\mathfrak{gl}(n)$ and its relation to our modified form $\widetilde B$;
(ii) the standard intertwining between the $\mathrm{SO}(3)$ vector action and the adjoint action on $\mathfrak{so}(3)$.

\subsection{The classical Killing form on $\mathfrak{gl}(n)$}
\label{app:killing_gln}

\begin{proposition}[Closed form and degeneracy]
\label{prop:killing_gln}
For $X,Y\in \mathfrak{gl}(n)$,
\begin{equation}
B_{\mathfrak{gl}(n)}(X,Y)=2n\,\operatorname{tr}(XY)-2\,\operatorname{tr}(X)\operatorname{tr}(Y).
\label{eq:killing_gln}
\end{equation}
In particular, $B_{\mathfrak{gl}(n)}$ is degenerate, and $\mathbb RI$ lies in its radical.
\end{proposition}
\begin{proof}[Proof sketch]
The identity \eqref{eq:killing_gln} is standard for $\mathfrak{gl}(n)$ (up to an overall normalization convention).
Degeneracy follows since $\mathrm{ad}_{\lambda I}=0$ for all $\lambda I\in\mathbb RI$, hence
$B_{\mathfrak{gl}(n)}(\lambda I,\cdot)=\operatorname{tr}(\mathrm{ad}_{\lambda I}\circ \mathrm{ad}_{\cdot})=0$.
\end{proof}

\begin{proposition}[Restriction to $\mathfrak{sl}(n)$]
\label{prop:killing_sln_restriction}
If $X,Y\in\mathfrak{sl}(n)$, then
\begin{equation}
B_{\mathfrak{gl}(n)}(X,Y)=2n\,\operatorname{tr}(XY),
\end{equation}
which coincides with the Killing form on the semisimple ideal $\mathfrak{sl}(n)$ under the same normalization.
\end{proposition}
\begin{proof}
For $X,Y\in\mathfrak{sl}(n)$, $\operatorname{tr}(X)=\operatorname{tr}(Y)=0$ and \eqref{eq:killing_gln} reduces to $2n\,\operatorname{tr}(XY)$.
\end{proof}

%------------------------------------------------------------
\subsection{Our modified $\mathrm{Ad}$-invariant form on $\mathfrak{gl}(n)$}
\label{app:Bt_gl}

In the main text we use
\begin{equation}
\widetilde B(X,Y)=2n\,\operatorname{tr}(XY)-\operatorname{tr}(X)\operatorname{tr}(Y).
\label{eq:Bt_gl}
\end{equation}

\begin{proposition}[Relation to $B_{\mathfrak{gl}(n)}$]
\label{prop:Bt_vs_killing}
With $B_{\mathfrak{gl}(n)}$ as in \eqref{eq:killing_gln},
\begin{equation}
\widetilde B(X,Y)=B_{\mathfrak{gl}(n)}(X,Y)+\operatorname{tr}(X)\operatorname{tr}(Y).
\end{equation}
\end{proposition}
\begin{proof}
Subtract \eqref{eq:killing_gln} from \eqref{eq:Bt_gl}.
\end{proof}

\begin{proposition}[Orthogonal splitting]
\label{prop:Bt_split}
Write $X=X_0+\frac{1}{n}\operatorname{tr}(X)I$ and $Y=Y_0+\frac{1}{n}\operatorname{tr}(Y)I$ with $X_0,Y_0\in\mathfrak{sl}(n)$.
Then
\begin{equation}
\widetilde B(X,Y)=2n\,\operatorname{tr}(X_0Y_0)+\operatorname{tr}(X)\operatorname{tr}(Y),
\end{equation}
and $\mathbb RI$ and $\mathfrak{sl}(n)$ are $\widetilde B$-orthogonal.
\end{proposition}
\begin{proof}
Using $\operatorname{tr}(I)=n$ and $\operatorname{tr}(X_0)=\operatorname{tr}(Y_0)=0$,
\begin{equation}
\operatorname{tr}(XY)=\operatorname{tr}(X_0Y_0)+\frac{1}{n}\operatorname{tr}(X)\operatorname{tr}(Y).
\end{equation}
Substitute into \eqref{eq:Bt_gl}. Orthogonality follows by setting $X=\lambda I$ and $Y\in\mathfrak{sl}(n)$.
\end{proof}

\begin{proposition}[$\mathrm{Ad}$-invariance]
\label{prop:Bt_ad_invariant}
For all $g\in\mathrm{GL}(n)$ and $X,Y\in\mathfrak{gl}(n)$,
\begin{equation}
\widetilde B(\mathrm{Ad}_g X,\mathrm{Ad}_g Y)=\widetilde B(X,Y).
\end{equation}
\end{proposition}
\begin{proof}
Use $\mathrm{Ad}_g X=gXg^{-1}$, cyclicity $\operatorname{tr}(gXYg^{-1})=\operatorname{tr}(XY)$, and $\operatorname{tr}(gXg^{-1})=\operatorname{tr}(X)$ in \eqref{eq:Bt_gl}.
\end{proof}

\begin{proposition}[$\widetilde B$ as an $\mathrm{Ad}$-invariant non-degenerate completion on the center]
\label{prop:Bt_completion_center}
The classical Killing form $B_{\mathfrak{gl}(n)}$ is degenerate with radical containing the center
$\mathfrak z(\mathfrak{gl}(n))=\mathbb RI$.
The modified form $\widetilde B$ in \eqref{eq:Bt_gl} is $\mathrm{Ad}$-invariant and non-degenerate, and it agrees with
$B_{\mathfrak{gl}(n)}$ on the semisimple ideal $\mathfrak{sl}(n)$ while supplying a non-degenerate inner product on the center:
\begin{equation}
\widetilde B|_{\mathfrak{sl}(n)\times \mathfrak{sl}(n)} = B_{\mathfrak{sl}(n)},\qquad
\widetilde B(\lambda I,\mu I)=n^2\,\lambda\mu \quad (\lambda,\mu\in\mathbb R).
\end{equation}
In particular, $\widetilde B$ resolves the degeneracy of $B_{\mathfrak{gl}(n)}$ precisely along $\mathbb RI$ without altering the semisimple part.
\end{proposition}

\begin{proof}
$\mathrm{Ad}$-invariance is Proposition~\ref{prop:Bt_ad_invariant}.
For $\lambda I,\mu I\in\mathbb RI$, \eqref{eq:Bt_gl} gives
$\widetilde B(\lambda I,\mu I)=2n\,\tr(\lambda\mu I)-\tr(\lambda I)\tr(\mu I)
=2n(\lambda\mu n)-(\lambda n)(\mu n)=n^2\lambda\mu$, which is non-degenerate.
The restriction to $\mathfrak{sl}(n)$ follows since $\tr(X)=0$ on $\mathfrak{sl}(n)$, so \eqref{eq:Bt_gl} reduces to $2n\,\tr(XY)$, which matches Proposition~\ref{prop:killing_sln_restriction} under the same normalization.
\end{proof}

%------------------------------------------------------------
\subsection{$\mathfrak{so}(3)\simeq\mathbb R^3$: vector action vs.\ adjoint action}
\label{app:so3_hat}

Let $\hat{\cdot}:\mathbb R^3\to\mathfrak{so}(3)$ be the standard hat map defined by $\hat v\,w=v\times w$, with inverse $(\cdot)^\vee:\mathfrak{so}(3)\to\mathbb R^3$.

\begin{lemma}[Intertwining identity]
\label{lem:so3_hat_adjoint}
For all $R\in\mathrm{SO}(3)$ and $v\in\mathbb R^3$,
\begin{equation}
\widehat{Rv}=R\hat v R^\top=\mathrm{Ad}_R(\hat v),
\qquad \mathrm{Ad}_R(X)=RXR^{-1}=RXR^\top.
\end{equation}
Consequently, if $F:\mathfrak{so}(3)\to\mathfrak{so}(3)$ is $\mathrm{Ad}$-equivariant, then
$f(v):=\bigl(F(\hat v)\bigr)^\vee$ is left $\mathrm{SO}(3)$-equivariant: $f(Rv)=Rf(v)$.
\end{lemma}
\begin{proof}
For any $w\in\mathbb R^3$, using $(Ra)\times(Rb)=R(a\times b)$,
\begin{equation}
(R\hat vR^\top)w
=R\hat v(R^\top w)
=R\bigl(v\times(R^\top w)\bigr)
=(Rv)\times w
=\widehat{Rv}\,w.
\end{equation}
Thus $\widehat{Rv}=R\hat vR^\top$. The equivariance statement follows by applying $(\cdot)^\vee$ and using $\widehat{Ra}=R\hat aR^\top$.
\end{proof}

\begin{remark}[Trace contraction recovers the Vector Neuron inner product]
\label{rem:so3_trace_dot}
Under the hat identification $\hat{\cdot}:\mathbb R^3\to\mathfrak{so}(3)$,
\begin{equation}
\hat v=
\begin{bmatrix}
0 & -v_3 & v_2\\
v_3 & 0 & -v_1\\
-v_2 & v_1 & 0
\end{bmatrix},
\qquad
\tr(\hat v\,\hat w)=-2\,v^\top w
\quad (v,w\in\mathbb R^3).
\end{equation}
Hence the $\mathrm{Ad}$-invariant trace pairing on $\mathfrak{so}(3)$,
$\langle X,Y\rangle_{\mathrm{tr}}:=-\tfrac12\tr(XY)$,
corresponds exactly to the Euclidean dot product on $\mathbb R^3$:
$\langle \hat v,\hat w\rangle_{\mathrm{tr}}=v^\top w$.
In particular, $\mathrm{Ad}$-invariant scalar contractions on $\mathfrak{so}(3)$ recover the invariants used by Vector Neurons ~\cite{deng2021vector}(up to a fixed global scaling, which can be absorbed by adjacent learnable linear maps or normalization).
\end{remark}

%------------------------------------------------------------
% Providing proofs of key theorems
\section{Proofs of Key Theorems}
\label{app:bilinear_details}
In this section, we provide a generalized treatment and proofs for the modified bilinear form $\widetilde{B}$ introduced in Section~\ref{sec:architecture} and Appendix~\ref{app:connections}. Our discussion here establishes the formal properties of $\widetilde{B}$ for any real reductive Lie algebra.

\subsection{Proof of Non-degeneracy and Ad-invariance of Modified Killing Form $B_e$}
Let $\mathfrak g$ be a real reductive Lie algebra.

\begin{definition}[Reductive decomposition]
A Lie algebra $\mathfrak g$ is \emph{reductive} if
$\mathfrak g=\mathfrak z(\mathfrak g)\oplus[\mathfrak g,\mathfrak g]$,
where $\mathfrak z(\mathfrak g)$ is the center and $[\mathfrak g,\mathfrak g]$ is semisimple.
This decomposition is canonical (both summands are ideals).
\end{definition}

\begin{definition}[Modified Killing form on a reductive Lie algebra]

Fix any symmetric, positive–definite inner product
$\langle\cdot,\cdot\rangle_{\mathfrak z}$ on $\mathfrak z(\mathfrak g)$, and let
$B$ denote the Killing form on the semisimple ideal $[\mathfrak g,\mathfrak g]$.
For $Z_i\in\mathfrak z(\mathfrak g)$ and $X_i\in[\mathfrak g,\mathfrak g]$ define
\begin{equation}
\label{eq:def-modB}
\widetilde B(Z_1{+}X_1,\,Z_2{+}X_2)
\;:=\;
\langle Z_1,Z_2\rangle_{\mathfrak z}\;+\;B(X_1,X_2).
\end{equation}
\end{definition}

\begin{remark}[Canonicity]
On $[\mathfrak g,\mathfrak g]$ the restriction (Killing form) is canonical.
On $\mathfrak z(\mathfrak g)$ there is no canonical choice; any $\mathrm{Ad}$-invariant positive-definite inner product works. 
The choice we make in the case of $\mathfrak{gl}(n)$ ensures that it agrees with the Killing form on the semisimple part $\mathfrak{sl}(n)$, and the center $\mathbb{R}I$ is normalized by a natural trace scale.
\end{remark}

\begin{proposition}[Block–orthogonality and restrictions]
\label{lem:block}
With notation as above,
\begin{equation}
\widetilde B\big(\mathfrak z(\mathfrak g),[\mathfrak g,\mathfrak g]\big)=0,\qquad
\widetilde B|_{\mathfrak z(\mathfrak g)}=\langle\cdot,\cdot\rangle_{\mathfrak z},\qquad
\widetilde B|_{[\mathfrak g,\mathfrak g]}=B.
\end{equation}
\end{proposition}

\begin{proposition}[Non–degeneracy]
\label{prop:nondeg}
$\widetilde B$ is nondegenerate on $\mathfrak g$.
\end{proposition}
\begin{proof}
Let $X=Z+W$ with $Z\in\mathfrak z(\mathfrak g)$ and $W\in[\mathfrak g,\mathfrak g]$.
If $\widetilde B(X,\cdot)\equiv 0$, then testing against
$Y\in\mathfrak z(\mathfrak g)$ yields $\langle Z,Y\rangle_{\mathfrak z}=0$ for all $Y$,
hence $Z=0$; testing against $Y\in[\mathfrak g,\mathfrak g]$ yields
$B(W,Y)=0$ for all $Y$, hence $W=0$ by the non–degeneracy of $B$ on the semisimple ideal.
Thus $X=0$.
\end{proof}

\begin{proposition}[$\mathrm{Ad}$–invariance on the identity component]
\label{prop:Ad-inv}
$\widetilde B$ is $\mathrm{ad}$–invariant:
\begin{equation}
\widetilde B([X,Y],Z)+\widetilde B(Y,[X,Z])=0
\qquad\text{for all }X,Y,Z\in\mathfrak g,
\end{equation}
and hence $\widetilde B(\mathrm{Ad}_g Y,\mathrm{Ad}_g Z)=\widetilde B(Y,Z)$ for all $g$ in the
identity component $G^\circ$.
\end{proposition}

\begin{proof}
The restriction to $[\mathfrak g,\mathfrak g]$ equals $B$, which is $\mathrm{ad}$–invariant.
If $Z\in\mathfrak z(\mathfrak g)$ then $[X,Z]=0$ for all $X$, so any bilinear
form on $\mathfrak z(\mathfrak g)$ is automatically $\mathrm{ad}$–invariant.
Using Proposition~\ref{lem:block} and bilinearity gives the displayed identity.
Equivalence with $\mathrm{Ad}$–invariance on $G^\circ$ follows by integrating the
infinitesimal relation along paths in $G^\circ$.
\end{proof}

\begin{remark}[Invariance for nonconnected groups]
\label{rem:avg}
In case the group is nonconnected, and one desires invariance under the full group $G$ (not just $G^\circ$). The
component group $\Gamma=G/G^\circ$ acts linearly on $\mathfrak z(\mathfrak g)$. In all practical cases, $\Gamma$ will be a finite group.
Then averaging any positive–definite $\langle\cdot,\cdot\rangle_{\mathfrak z}$ over $ \Gamma$ yields an $\mathrm{Ad}(G)$–invariant inner product
on the center:
\begin{equation}
\langle Z_1,Z_2\rangle_{\mathfrak z}^{\mathrm{avg}}
=\frac{1}{|\Gamma|}\sum_{\gamma\in\Gamma}
\big\langle \mathrm{Ad}_\gamma Z_1,\,\mathrm{Ad}_\gamma Z_2\big\rangle_{\mathfrak z}.
\end{equation}
  Replacing $\langle\cdot,\cdot\rangle_{\mathfrak z}$ by
$\langle\cdot,\cdot\rangle_{\mathfrak z}^{\mathrm{avg}}$ in Equation~\ref{eq:def-modB} makes
$\widetilde B$ invariant under all of $G$.
\end{remark}

%------------------------------------------------------------

\section{Summary of Key Lie Groups and Algebras}
\label{app:groups_summary}

To provide a comprehensive overview of the geometric structures discussed in this work, Table~\ref{tab:groups_summary_appendix} summarizes the Lie groups $G$, their associated Lie algebras $\mathfrak{g}$, and their application domains. We categorize each group by its algebraic properties—specifically distinguishing between reductive, semisimple, and compact types—and detail their specific roles within both the broader literature and our experimental framework. This unified taxonomy situates ReLNs as a general-purpose backbone capable of handling a wide array of linear symmetries encountered in scientific machine learning.

\begin{table*}[h]
\centering
\scriptsize
\caption{Comprehensive survey of Lie groups, algebras, and their applications. We highlight the algebraic classifications that dictate the choice of invariant bilinear forms and contrast general literature examples with our specific experimental tasks.}
\label{tab:groups_summary_appendix}
\begin{tabularx}{\textwidth}{l l l X X}
\toprule
\textbf{Group ($G$)} & \textbf{Algebra ($\mathfrak{g}$)} & \textbf{Algebraic Type} & \textbf{General Applications \& Refs} & \textbf{Relevance to ReLN Framework} \\
\midrule

\textbf{General Linear} & $\mathfrak{gl}(n)$ & Reductive & General linear transformations, stress-strain tensors \citep{basu2024g, finzi2021practical}. & \textbf{Core Domain.} Primary target of ReLNs; resolves the Killing form degeneracy on $\mathfrak{gl}(n)$. \\ \addlinespace

\textbf{Special Linear} & $\mathfrak{sl}(3)$ & Semisimple & Homography classification, 3D vision, SLAM \citep{lin2023lie, finzi2021practical}. & \textbf{Algebraic Benchmark.} Validates generalization on semisimple non-compact subgroups (Section ~\ref{subsec:platonic}). \\ \addlinespace

\textbf{Symplectic} & $\mathfrak{sp}(4)$ & Semisimple & Hamiltonian mechanics, phase space dynamics \citep{lin2023lie, finzi2021practical}. & \textbf{Algebraic Benchmark.} Tests the ability to learn highly nonlinear invariants on conserved systems (Section ~\ref{subsec:sp4}). \\ \addlinespace

\textbf{Special Orthogonal} & $\mathfrak{so}(3)$ & Compact & 3D Point clouds, state estimation, rigid body rotations \citep{deng2021vector, son2024intuitive, thomas2018tensor}. & \textbf{Robotics Application.} Lifts 3D velocity vectors into $\mathfrak{so}(3) \subset \mathfrak{gl}(3)$ (Sections~\ref{subsec:covariance} and~\ref{sec:3dgs_experiment}). \\ \addlinespace

\textbf{Lorentz Group} & $\mathfrak{so}(1,3)^+$ & Semisimple & Particle physics, jet tagging, relativistic collisions \citep{bogatskiy2020lorentz, batatia2023a}. & \textbf{Physics Application.} 4-momenta embedded into $\mathfrak{gl}(5)$ to unify left-action with adjoint action (Section ~\ref{app:lorentz_results}. \\ \addlinespace

\textbf{Special Unitary} & $\mathfrak{su}(3)$ & Compact & Quantum Chromodynamics (QCD), lattice gauge theory \citep{favoni2022lattice}. & \textbf{General Example.} Illustrates framework applicability to complex-valued compact groups. \\ \addlinespace

\textbf{SPD Manifold}$^\dagger$ & $\mathrm{SPD}(n)$ & Riemannian & Geometric uncertainty, inertia tensors, diffusion MRI \citep{ma2025large, magnus1985differentiating}. & \textbf{Geometric Uncertainty.} Covariances mapped to $\mathfrak{gl}(n)$ (Sections~\ref{subsec:covariance} and~\ref{sec:3dgs_experiment}). \\

\bottomrule
\end{tabularx}
\begin{flushleft}
\scriptsize $^\dagger$ Note: $\mathrm{SPD}(n)$ is not a group but a manifold representable as the quotient space $\mathrm{GL}(n)/\mathrm{O}(n)$. We process it by mapping to the linear subspace $\mathfrak{sym}(n) \subset \mathfrak{gl}(n)$ via the matrix logarithm \citep{ma2025large}.
\end{flushleft}
\end{table*}

\section{Detailed Layer Formulations}
\label{app:layer_details}
This section provides the precise mathematical definitions and equivariance proofs for the ReLN architecture. To maintain consistency with the main text, we denote the input tensor as $x \in \mathbb{R}^{B \times K \times C}$, where $B$ is the batch size, $K = \dim\mathfrak{g}$ is the geometric dimension, and $C$ is the number of feature channels. Each column $x_{b, :, c} \in \mathbb{R}^K$ corresponds to a matrix $X_{b,c} \in \mathfrak{g}$ via the vee/hat isomorphism.

\subsection{Equivariant Linear Layer}
The ReLN-Linear layer applies a linear map to the channel dimension (the last axis). For an input $x \in \mathbb{R}^{B \times K \times C}$ and weights $W \in \mathbb{R}^{C \times C'}$, the operation is defined as:
\begin{equation}
    f_{\mathrm{ReLN-Lin}}(x; W) = xW \in \mathbb{R}^{B \times K \times C'}.
\end{equation}
We omit any bias term to preserve exact equivariance. 

\paragraph{Proof of Equivariance.}
The group action, $\mathrm{Ad}_g$ (defined in Equation~\ref{eq:adjoint-group}), is a linear map that multiplies each feature channel from the left. The weight matrix $W$ multiplies the channel dimension from the right. These operations commute, ensuring strict $G$-equivariance for any $g \in G$:
\begin{equation}
\begin{split}
    f_{\mathrm{ReLN-Lin}}(\mathrm{Ad}_g(x); W) &= (\mathrm{Ad}_g x) W \\
    &= \mathrm{Ad}_g (xW) \\
    &= \mathrm{Ad}_g (f_{\mathrm{ReLN-Lin}}(x; W)).
\end{split}
\end{equation}

\subsection{Equivariant Nonlinearities}
Standard pointwise activations break equivariance under non-orthogonal transforms. We introduce two equivariant alternatives.

\paragraph{ReLN-ReLU.}
This layer rectifies a feature based on its alignment with a learnable direction. Given the input $x \in \mathbb{R}^{B \times K \times C}$, we first compute equivariant directions $d = f_{\mathrm{ReLN-Lin}}(x; U)$ where $U \in \mathbb{R}^{C \times C}$. The nonlinearity is then applied channel-wise as a scalar-gated update:
\begin{equation}
    f_{\mathrm{ReLN-ReLU}}(x)_{b,:,c} =
    \begin{cases}
        x_{b,:,c}, & \text{if } \widetilde{B}(x_{b,:,c}^\wedge, d_{b,:,c}^\wedge) \le 0, \\
        x_{b,:,c} + \sigma\big(\widetilde{B}(x_{b,:,c}^\wedge, d_{b,:,c}^\wedge)\big) d_{b,:,c}, & \text{otherwise},
    \end{cases}
\end{equation}
where $\sigma(t) = \max(0, t)$ denotes the standard ReLU activation function acting as an invariant scalar gate. Since the modified bilinear form $\widetilde{B}$ is $\mathrm{Ad}$-invariant, the gate $\sigma(\widetilde{B}(\cdot, \cdot))$ is group-invariant. Furthermore, because $d$ transforms equivariantly with $x$, the resulting vector update preserves strict $\mathrm{Ad}_g$-equivariance. The leaky variant $f_{\mathrm{ReLN-Leaky ReLU}}(x) = \alpha x + (1-\alpha) f_{\mathrm{ReLN-ReLU}}(x)$ follows directly.

\paragraph{ReLN-Bracket (Lie-bracket nonlinearity).}
This layer uses the matrix commutator, a natural $\mathrm{Ad}$-equivariant primitive, to create learnable interactions between channels. Let the input be a batch of features represented by their vector coordinates, $x \in \mathbb{R}^{B \times K \times C_{\mathrm{in}}}$. 

First, two independent linear maps with weights $W_a, W_b \in \mathbb{R}^{C_{\mathrm{in}} \times C_{\mathrm{out}}}$ produce intermediate tensors $u = xW_a$ and $v = xW_b$. The Lie bracket is then computed channel-wise between the corresponding feature vectors of $u$ and $v$ to produce an update tensor $\Delta x \in \mathbb{R}^{B \times K \times C_{\mathrm{out}}}$:
\begin{equation}
    (\Delta x)_{b, :, c'} = \left[ (u_{b, :, c'})^\wedge, (v_{b, :, c'})^\wedge \right]^\vee.
\end{equation}

This update is added to the input via a residual connection (requiring $C_{\mathrm{in}} = C_{\mathrm{out}}$):
\begin{equation}
    f_{\mathrm{ReLN-Bracket}}(x) = x + \Delta x.
\end{equation}

Each step in this process is equivariant under the adjoint action, making the entire block exactly $\mathrm{Ad}_g$-equivariant: $f_{\mathrm{ReLN-Bracket}}(\mathrm{Ad}_g x) = \mathrm{Ad}_g(f_{\mathrm{ReLN-Bracket}}(x))$ for all $g \in G$.

\section{Theoretical Complexity and Scaling Analysis}
\label{app:complexity}

To address the operational properties of our architecture, we provide an explicit Big-$\mathcal{O}$ complexity analysis connecting the feature channel width $C$ and the Lie algebra matrix dimension $n$. We further validate this analysis with empirical benchmarks.

\paragraph{1. Theoretical Scaling with Matrix Dimension $n$.}
We fix the \textbf{total number of scalars per token} to $C$ to ensure a fair comparison. The effective channel width $M$ is scaled such that $M \times K = C$.
\begin{itemize}
    \item \textbf{Standard MLP ($K=1$):} Uses full width $M = C$.
    \item \textbf{ReLN ($K=n^2$):} Operates on $\mathfrak{gl}(n)$ matrices. Effective width $M = C/n^2$.
\end{itemize}

Let $n$ be the dimension of the matrices (e.g., $n=3$ for $\mathfrak{gl}(3)$).
\begin{itemize}
    \item \textbf{Linear Layer:} The layer performs $K$ independent matrix multiplications of size $M \times M$.
    \begin{equation}
        \text{Cost}_{\text{linear}} = K \cdot M^2 = n^2 \cdot \left(\frac{C}{n^2}\right)^2 = \frac{C^2}{n^2}.
    \end{equation}
    The dominant quadratic cost \textbf{decreases} as the matrix dimension $n$ increases. For $n=3$, ReLN reduces the Linear Layer FLOPs by a factor of $1/9$ compared to an MLP ($n=1$).
    
    \item \textbf{Lie Bracket Nonlinearity:} This operation computes $[X, Y] = XY - YX$, involving $n \times n$ matrix multiplications.
    \begin{equation}
        \text{Cost}_{\text{bracket}} = M \cdot \mathcal{O}(n^3) = \frac{C}{n^2} \cdot \mathcal{O}(n^3) = \mathcal{O}(C \cdot n).
    \end{equation}
    In deep learning contexts where $C$ is large (e.g., $C \ge 128$), the quadratic term $\mathcal{O}(C^2/n^2)$ dominates the linear term $\mathcal{O}(Cn)$.
\end{itemize}

\paragraph{2. Layer-wise Complexity Comparison.}
Table~\ref{tab:complexity_comparison} presents the theoretical breakdown. ReLN-Bracket reduces the total block cost to $\approx 12.5\%$ of the baseline.

\begin{table}[h]
    \centering
    \caption{\textbf{Theoretical Complexity Analysis.} Comparison under a fixed feature budget $C$. By utilizing a higher-dimensional geometric space ($n=3, K=9$), ReLN reduces the effective channel width to $C/9$, resulting in a quadratic reduction in the dominant Linear Layer cost.}
    \label{tab:complexity_comparison}
    \vspace{2mm}
    % \resizebox{\textwidth}{!}{%
    \small 
    \begin{tabular}{l l c c c c}
        \toprule
        \textbf{Layer Type} & \textbf{Model} & \textbf{Eff. Channels ($M$)} & \textbf{Parameters} & \textbf{FLOPs} & \textbf{Relative Cost} \\
        \midrule
        \multirow{3}{*}{\textbf{Linear Layer}} 
        & MLP ($n=1$) & $C$ & $C^2$ & $C^2$ & $1.00\times$ \\
        & VN ($n\approx 1.7$) & $C/3$ & $0.33 C^2$ & $0.33 C^2$ & $0.33\times$ \\
        & \textbf{ReLN ($n=3$)} & $\mathbf{C/9}$ & $\mathbf{0.11 C^2}$ & $\mathbf{0.11 C^2}$ & $\mathbf{0.11\times}$ \\
        \midrule
        \multirow{4}{*}{\textbf{Nonlinearity}} 
        & ReLU (MLP) & $C$ & 0 & $\mathcal{O}(C)$ & $\approx 0.00\times$ \\
        & VN-ReLU & $C/3$ & $0.11 C^2$ & $0.11 C^2$ & $0.11\times$ \\
        & ReLN-ReLU & $C/9$ & $0.01 C^2$ & $0.01 C^2$ & $0.01\times$ \\
        & \textbf{ReLN-Bracket} & $\mathbf{C/9}$ & $\mathbf{0.025 C^2}$ & $\mathbf{0.025 C^2}$ & $\mathbf{0.03\times}$ \\
        \bottomrule
    \end{tabular}
    % }
\end{table}

\section{Experimental Details}
\label{app:exp_details}
Training and evaluation for all presented experiments, Platonic solid classification, invariant function regression, top tagging, and drone state estimation, were conducted on a single NVIDIA GeForce RTX 4090 GPU.

\subsection{Model Architectures and Implementation Details}
\label{app:architectures}

Across all experiments, our proposed ReLN models are constructed by stacking ReLN-Linear, ReLN-ReLU, and ReLN-Bracket layers. The specific number of layers and channel widths are adapted for each task to ensure a fair comparison with baseline models in terms of parameter count.

\paragraph{Algebraic Benchmarks ($\mathfrak{sl}(3)$ and $\mathfrak{sp}(4)$).}
For the Platonic solid classification and $\mathfrak{sp}(4)$ invariant regression tasks, our ReLN model directly adopts the architecture used by the Lie Neurons benchmark model from \citet{lin2023lie}. The primary modification is the replacement of their Killing form-based nonlinearity and invariant layers with our proposed non-degenerate bilinear form $\widetilde{B}$ (Eq.~\ref{eq:killing-gl-concrete}). This setup allows for a direct comparison of the impact of the bilinear form, as all other architectural hyperparameters are kept identical to the baseline.

\paragraph{Top Tagging.}
For the Top-Tagging task, our model is a modification of the LorentzNet architecture~\citep{gong2022efficient}. We adapt its Lorentz Group Equivariant Blocks (LGEBs) by replacing the invariant feature computation with our proposed bilinear form. A detailed description of the architecture, our modifications, and training protocol is provided in Appendix~\ref{app:top_tagging}.

\paragraph{Drone State Estimation.}
In this task, we compare our ReLN model against two baseline families: a non-equivariant 1D ResNet and an equivariant Vector Neurons (VN) model. The specific implementation details and architectural choices for each model are provided next in Appendix~\ref{app:drone_exp_details}.

\subsection{Platonic Solid Classification on $\mathfrak{sl}(3)$}
\label{app:hyperparams}

\paragraph{Overview.}
All experiments evaluate classification of Platonic solids (tetrahedron, octahedron, icosahedron) from inter-face homographies computed in the image plane. For each model-family we train 5 independent runs with different random seeds and report mean ± standard deviation. Training uses fixed object and camera poses; at test time we report results on the in-distribution (ID) split and the rotated-camera (RC) split (RC applies ten random \(\mathrm{SO}(3)\) rotations to the camera frame). The `MLP (wider)' denotes a capacity-matched ($\approx$ 2× parameters) MLP used for a fairer comparison.

\begin{table}[h]
    \scriptsize
  \centering
  \caption{Common training hyperparameters (used across model families unless noted).}
  \label{tab:hyper_common}
  \begin{tabular}{l l}
    \toprule
    Hyperparameter & Value \\
    \midrule
    Optimizer & Adam \\
    Batch size & 100 \\
    Number of independent runs (seeds) & 5 \\
    Max epochs / stopping criterion & 5000 epochs \\
    Data augmentation (train) & Random camera rotations applied to training examples when enabled \\
    RC evaluation & 500 random \(\mathrm{SO}(3)\) rotations applied to each test example \\
    Metric reported & Classification accuracy (mean ± std across runs) \\
    \bottomrule
  \end{tabular}
\end{table}

\begin{table}[h]
  \centering
    \scriptsize
  \caption{Model-specific hyperparameters and implementation notes.}
  \label{tab:hyper_model_specific}
  \begin{tabular}{l l l}
    \toprule
    Model family & Key choices & Notes \\
    \midrule
    Latent Feature Size (MLP Baseline)  & 256 & As in \citet{lin2023lie}. \\
    Latent Feature Size (MLP Wider) & 386 & Increased width total parameters $\approx 2\times$ baseline.\\
    Learning rate (MLP models) & $1\times 10^{-4}$ & \\
    Learning rate (ReLNs/Lie Neurons models) & $3\times 10^{-6}$ & Lower LR chosen for stable training \\
    \bottomrule
  \end{tabular}
\end{table}

%%%%%%%%%%%%%%% Lorentz Group Benchmark %%%%%%%%%%%%%%%
\section{Top Tagging Experiment: Framework, Proof, and Implementation}
\label{app:top_tagging}

This appendix provides details for our Lorentz-equivariant jet tagging study. We first summarize the task and results, then present the Lorentz-compatible embedding and its equivariance proof, and finally describe the model and training protocol.

\subsection{Particle Physics with Lorentz Group $\mathrm{SO}^+(1,3)$ Equivariance}
\label{app:lorentz_results}

We evaluate on the Top-Tagging benchmark~\citep{kasieczka2019machine}, which classifies particle jets originating from top quarks against a Quantum Chromodynamics background. Since relativistic collisions respect spacetime symmetries, the model should be equivariant under the proper orthochronous Lorentz group $\mathrm{SO}^+(1,3)$. To express this as an adjoint action within our Lie-algebraic framework, we embed each four-momentum $p\in\mathbb{R}^{1,3}$ into $\mathfrak{gl}(5)$ via
\begin{equation}
\varphi(p) \;=\;
\begin{bmatrix}
0_{4\times4} & p \\
p^\top \eta & 0
\end{bmatrix},
\qquad
\eta=\mathrm{diag}(-1,1,1,1),
\label{eq:lorentz_embedding_app}
\end{equation}
and show in Appendix~\ref{app:lorentz_proof} that $\varphi(\Lambda p)=\mathrm{Ad}_{\mathrm{diag}(\Lambda,1)}\big(\varphi(p)\big)$ for $\Lambda\in\mathrm{SO}^+(1,3)$.

For comparison, we modify LorentzNet by replacing its invariant feature computation with our bilinear-form construction and additionally report a parameter-matched LorentzNet baseline. Table~\ref{tab:top_tagging_results} shows that ReLN achieves performance comparable to the parameter-matched LorentzNet across Acc./AUC and Rej@30\%, with differences within the reported uncertainty (i.e., overlapping standard deviations).

% --- TABLE FOR JET TAGGING RESULTS ---
\begin{table}[t]
    \centering
    \caption{
        Comparison of performance on the Top-Tagging dataset.
        ReLN results include both the parameter-efficient version (84k) and the parameter-matched version (224k).
        Rej@30\% denotes background rejection at 30\% signal efficiency (higher is better).
        Benchmark scores are as reported in the original publications.}
    \label{tab:top_tagging_results}
    \scriptsize
    \setlength{\tabcolsep}{3pt}
    \resizebox{0.7\linewidth}{!}{
        \begin{tabular}{l l c c c l}
            \toprule
            \textbf{Architecture} & \textbf{\#Params} & \textbf{Acc.} & \textbf{AUC} & \textbf{Rej@30\%} & \textbf{Reference} \\
            \midrule
            PELICAN                 & 45k      & 0.943 & 0.987 & $2289 \pm 204$ & \citet{bogatskiy2022pelican} \\
            LorentzNet (orig.)      & 224k     & 0.942 & 0.987 & $2195 \pm 173$ & \citet{gong2022efficient} \\
            LorentzNet (reprod.)    & 84k      & 0.942 & 0.987 & $1821 \pm 94$  & Our reprod. \\
            LGN                     & 4.5k     & 0.929 & 0.964 & $435 \pm 95$   & \citet{bogatskiy2020lorentz} \\
            BIP                     & 4k       & 0.931 & 0.981 & $853 \pm 68$   & \citet{munoz2022boost} \\
            partT                   & 2.14M    & 0.940 & 0.986 & $1602 \pm 81$  & \citet{qu2022particle} \\
            ParticleNet             & 498k     & 0.938 & 0.985 & $1298 \pm 46$  & \citet{qu2020jet} \\
            EFN                     & 82k      & 0.927 & 0.979 & $633 \pm 31$   & \citet{komiske2019energy} \\
            TopoDNN                 & 59k      & 0.916 & 0.972 & $295 \pm 5$    & \citet{pearkes2017jet} \\
            LorentzMACE             & 228k     & 0.942 & 0.987 & $1935 \pm 85$  & \citet{batatia2023a} \\
            CGENN                   & 321K     & 0.942 & 0.987 & $2172$         & \citet{ruhe2023clifford} \\
            \midrule
            \textbf{Lie Neurons}    & 224k     & 0.941 & 0.985 & $1655 \pm 73$ & Our reprod. \\
            \textbf{ReLN (Ours)}    & 224k     & 0.942 & 0.987 & $2201 \pm 101$ & \\
            \textbf{ReLN (Ours)}    & 84k      & 0.942 & 0.987 & $1979 \pm 87$ & \\
            \bottomrule
        \end{tabular}
    }
\end{table}

\subsection{Group Action Equivariance via Embedding Map}
\label{app:lorentz_proof}

To process four-momenta within our Lie-algebraic framework, we require an embedding that translates the Lorentz action into an adjoint action on a matrix space.

\begin{definition}[Lorentz-Compatible Embedding]
Given $p \in \mathbb{R}^4$ and $\eta = \operatorname{diag}(-1, 1, 1, 1)$, define
\begin{equation}
    \varphi(p) =
    \begin{bmatrix}
    0_{4 \times 4} & p \\
    p^\top \eta & 0
    \end{bmatrix}.
    \label{eq:lorentz_embedding}
\end{equation}
\end{definition}

\begin{theorem}[Adjoint Equivariance]
For any $p \in \mathbb{R}^4$ and $\Lambda \in \mathrm{SO}^+(1,3)$, let $G = \operatorname{diag}(\Lambda, 1) \in \operatorname{GL}(5)$. Then
\begin{equation}
    \operatorname{Ad}_G(\varphi(p)) = G\varphi(p)G^{-1} = \varphi(\Lambda p).
    \label{eq:equivariance_condition}
\end{equation}
\end{theorem}

\begin{proof}
We compute the left-hand side (LHS) of Eq. \ref{eq:equivariance_condition}, which is the adjoint action:
\begin{align}
    \operatorname{Ad}_G(\varphi(p)) &= 
    \begin{bmatrix}
    \Lambda & 0 \\
    0 & 1
    \end{bmatrix}
    \begin{bmatrix}
    0 & p \\
    p^\top \eta & 0
    \end{bmatrix}
    \begin{bmatrix}
    \Lambda^{-1} & 0 \\
    0 & 1
    \end{bmatrix} 
    = 
    \begin{bmatrix}
    0 & \Lambda p \\
    p^\top \eta \Lambda^{-1} & 0
    \end{bmatrix}.
    \label{eq:lhs_result}
\end{align}
The right-hand side (RHS) is the lift of the transformed vector $\Lambda p$:
\begin{equation}
    \varphi(\Lambda p) = 
    \begin{bmatrix}
    0 & \Lambda p \\
    (\Lambda p)^\top \eta & 0
    \end{bmatrix}
    =
    \begin{bmatrix}
    0 & \Lambda p \\
    p^\top \Lambda^\top \eta & 0
    \end{bmatrix}.
    \label{eq:rhs_result}
\end{equation}
For the LHS and RHS to be equal, we must show that $\eta \Lambda^{-1} = \Lambda^\top \eta$. We start from the defining property of $\operatorname{SO}(1,3)$:
\begin{equation}
    \Lambda^\top \eta \Lambda = \eta.
    \label{eq:lorentz_property}
\end{equation}
Right-multiplying Eq. \ref{eq:lorentz_property} by $\Lambda^{-1}$ yields the desired identity:
\begin{align}
    (\Lambda^\top \eta \Lambda) \Lambda^{-1} &= \eta \Lambda^{-1} \implies \Lambda^\top \eta ( \Lambda \Lambda^{-1}) = \eta \Lambda^{-1} \implies \Lambda^\top \eta = \eta \Lambda^{-1}.
    \label{eq:final_identity}
\end{align}
Since the condition holds, the proof is complete.
\end{proof}

\begin{remark}[Generalization to Orthogonal Groups]
\label{rem:lifting_generalization}
This embedding technique is not limited to the Lorentz group and can be readily generalized to any orthogonal group $\operatorname{O}(n)$ or special orthogonal group $\operatorname{SO}(n)$. For instance, in applications involving 3D point clouds where the symmetry is $\operatorname{SO}(3)$, a vector $p \in \mathbb{R}^3$ would be embedded into the Lie algebra $\mathfrak{gl}(4)$ as:
\begin{equation}
\varphi(p) = 
\begin{bmatrix}
0_{3 \times 3} & p \\
p^\top & 0
\end{bmatrix}
\end{equation}
The proof of equivariance follows the same structure, using the property of orthogonal matrices, $R^\top R = I$ (which implies $R^{-1} = R^\top$), instead of the Minkowski metric identity. This highlights the broad applicability of our embedding strategy to any benchmark involving norm-preserving group transformations.
\end{remark}

\subsection{Experimental Implementation}
\paragraph{Dataset.}
The experiment uses the Top-Tagging dataset~\citep{kasieczka2019machine}, which contains 2 million simulated proton-proton collision events. The dataset was generated with Pythia, Delphos, and FastJet to model the ATLAS detector response. We use the standard 60\%/20\%/20\% splits for training, validation, and testing. Each jet is represented as a set of constituent particles,
each with four-momentum $p=(E,p_x,p_y,p_z)$.

\paragraph{Model.}
Our model leverages the established architecture of LorentzNet \citep{gong2022efficient}, utilizing its stack of Lorentz Group Equivariant Blocks (LGEBs) for message passing on the jet's particle cloud. While the original LorentzNet computes these features directly from the 4-momenta using the Minkowski inner product, our approach introduces a modified bilinear form-based feature extraction. We first embed each pair of 4-momenta, $p_i$ and $p_j$, from the Minkowski space $\mathbb{R}^{1,3}$ into the Lie algebra $\mathfrak{gl}(5)$ via the map $p \mapsto \varphi(p)$, as defined in Equation~\ref{eq:lorentz_embedding}. The invariant features for the message passing are then derived from the bilinear form, $\widetilde{B}(\cdot, \cdot)$, on this Lie algebraic space. The edge message $m_{ij}$ is thus constructed as:
\begin{equation}
    m_{ij} = \phi_e \Big( h_i, h_j, \psi\big(\widetilde B(\varphi(p_i), \varphi(p_i))\big), \psi\big(\widetilde B(\varphi(p_j), \varphi(p_j))\big), \psi\big(\widetilde B(\varphi(p_i), \varphi(p_j))\big) \Big)
    \label{eq:bilinear_form_message}
\end{equation}
where $h_i, h_j$ are scalar features, $\phi_e$ is an MLP, and $\psi$ is a stabilizing nonlinearity. As shown in the main results (Table \ref{tab:top_tagging_results}), this approach leads to an advantage in background rejection when compared against a parameter-matched LorentzNet baseline. The architectural differences are summarized in Table~\ref{tab:top_tagging_arch_details}.

\begin{table}[h!]
\centering
\scriptsize
\caption{Architectural comparison for the Top-Tagging task.}
\label{tab:top_tagging_arch_details}
\begin{tabular}{l l l l}
\toprule
\textbf{Component} & \textbf{LorentzNet (Original)} & \textbf{Param-matched Baseline} & \textbf{Ours (ReLN)} \\
\midrule
Number of LGEBs & 6 & 5 & 5 \\
Hidden feature dims & 72 & 48 & 48 \\
Edge feature computation & Minkowski inner prod. & Minkowski inner prod. & \textbf{Bilinear invariant form} \\
\bottomrule
\end{tabular}
\end{table}

\paragraph{Training Setup.} For a fair comparison, our training procedure closely follows the protocol established in the LorentzNet~\citep{gong2022efficient}. The model was trained for a total of 35 epochs on a NVIDIA RTX 4090 GPU. We used the {AdamW optimizer} with a weight decay of 0.01 and a batch size of 128, matching the total effective batch size from the reference work. The learning rate was managed by the paper's specific three-stage schedule: a 4-epoch linear warm-up to an initial rate of \(1 \times 10^{-3}\), followed by a 28-epoch \texttt{CosineAnnealingWarmRestarts} schedule, and a final 3-epoch exponential decay. After each epoch, the model with the highest validation accuracy was saved for final evaluation on the test set.

\section{Drone Experiment Details}
\label{app:drone_exp_details}

This appendix provides the technical details for the drone state estimation experiment, including the theoretical framework, dataset generation, model implementations, and formal proofs.

\subsection{Geometric Framework for Equivariant Covariance Processing}
\label{app:geo_framework}
Our approach leverages the geometry of symmetric positive-definite matrices. A covariance matrix $C$ is symmetric positive-definite, residing on the manifold $\mathrm{SPD}(3)$. A non-degenerate covariance matrix $C \in \mathrm{SPD}(n)$ represents the anisotropic stretching of a general linear map, as seen via the polar decomposition $A = QP$ with $Q \in \mathrm{O}(n)$ and $P \in \mathrm{SPD}(n)$. Equivalently, there is a homogeneous-space isomorphism:
$
\mathrm{SPD}(n) \;\cong\; \mathrm{GL}(n)/\mathrm{O}(n),
$
which motivates processing covariances in a $\mathrm{GL}(n)$-aware architecture.

While $\mathrm{SPD}(3)$ is not a Lie group, the matrix logarithm provides a canonical map to the vector space of symmetric matrices $\mathrm{Sym}(3)$, which is a linear subspace of $\mathfrak{gl}(3)$.
\begin{equation}
    \log : \mathrm{SPD}(3) \;\longrightarrow\; \mathrm{Sym}(3) \subset \mathfrak{gl}(3).
\end{equation}

This allows us to embed a geometric object from a curved manifold into a flat, Lie-algebra-compatible space. The following theorem proves that the congruence transformation on $C \in \mathrm{SPD}(n)$ becomes an adjoint action on its image $\log C \in \mathrm{Sym}(n)$, thus preserving the equivariant structure required by our model.

\begin{theorem}[Equivariance of the Logarithmic Map]
For any $C \in \mathrm{SPD}(n)$ and any rotation matrix $R \in \mathrm{SO}(n)$, the congruence transformation on $C$ corresponds to an adjoint action on its logarithm:
\begin{equation}
    \log(RCR^\top) = R(\log C)R^\top.
\end{equation}
\end{theorem}

\begin{proof}
The proof follows from the spectral theorem for real symmetric matrices.
\begin{enumerate}
    \item Let the eigendecomposition of $C$ be $C = V \Lambda V^\top$, where $V$ is an orthogonal matrix ($V^\top V = I$) of eigenvectors and $\Lambda$ is the diagonal matrix of corresponding positive eigenvalues.
    
    \item By definition, the matrix logarithm of $C$ is given by applying the logarithm to its eigenvalues:
    \begin{equation}
        \log C := V (\log \Lambda) V^\top
    \end{equation}
    where $\log \Lambda$ is the diagonal matrix of element-wise logarithms of the eigenvalues.

    \item Consider the transformed matrix $C' = RCR^\top$. Substituting the decomposition of $C$ yields:
    \begin{equation}
        C' = R(V \Lambda V^\top)R^\top = (RV) \Lambda (V^\top R^\top) = (RV) \Lambda (RV)^\top
    \end{equation}
    This is the eigendecomposition of $C'$, where the new orthogonal matrix of eigenvectors is $V' = RV$ and the eigenvalues $\Lambda$ are unchanged.

    \item Applying the definition of the matrix logarithm to $C'$ gives:
    \begin{equation}
        \log(C') = V' (\log \Lambda) (V')^\top = (RV) (\log \Lambda) (RV)^\top
    \end{equation}

    \item Rearranging the terms, we arrive at the desired identity:
    \begin{equation}
        \log(C') = R \left( V (\log \Lambda) V^\top \right) R^\top = R(\log C)R^\top
    \end{equation}
\end{enumerate}
\end{proof}

This identity is critical, as it confirms that our adjoint-equivariant network can process either the raw covariance $C$ or its logarithm $\log C$ while perfectly preserving the $\mathrm{SO}(3)$ symmetry.

In the \(\mathrm{SO}(3)\) regime used in our experiments, vectors (e.g., velocity \(\mathbf{v}\)) are represented in the Lie algebra \(\mathfrak{so}(3)\) so that the adjoint action coincides with ordinary rotation, \(\mathrm{Ad}_R(\mathbf{v})=R\mathbf{v}\). Conjugation then implements the covariance congruence \(C\mapsto RCR^\top\). Consequently, ReLNs realize \(\mathrm{SO}(3)\)-equivariance \emph{by construction}, avoiding the need for the model to learn these symmetries from data.

\subsection{Dataset Generation.}
\label{app:dataset}

We use the PyBullet engine to simulate 200 aggressive trajectories for a Crazyflie-like nano-quadrotor. To generate realistic measurements, the instantaneous velocity is corrupted by Gaussian noise,
$
\mathbf{v}_{\mathrm{noisy}} \sim \mathcal{N}(\mathbf{v}_{\mathrm{gt}},\,C_v),
$
where the covariance \(C_v\) varies with flight aggressiveness. The dataset provides time series of noisy velocities, ground-truth covariances, and ground-truth trajectories for evaluation.

\paragraph{Trajectory Generation.}
The procedure begins with the procedural generation of a sequence of 20 to 40 random 3D waypoints within a flight volume of approximately $170\text{m} \times 170\text{m} \times 60\text{m}$. The waypoints are sampled from a uniform distribution to create diverse flight paths. To mimic the complex dynamics of aggressive flight, each trajectory is randomly generated using a path with random wiggles or a path featuring high-speed spiral maneuvers. These discrete waypoints are then interpolated using a Catmull-Rom spline to create a smooth, $C^1$ continuous target trajectory, which is densely sampled at an 80\,Hz control frequency. Each of the 200 sequences results in a unique trajectory lasting approximately 2-4 minutes, totaling over 13 hours of simulated flight time. A sample generated trajectory is shown in Figure~\ref{fig:sample_trajectory}.

\begin{figure}[h]
    \centering
    \begin{subfigure}[b]{0.48\textwidth}
        \centering
        \includegraphics[width=\textwidth]{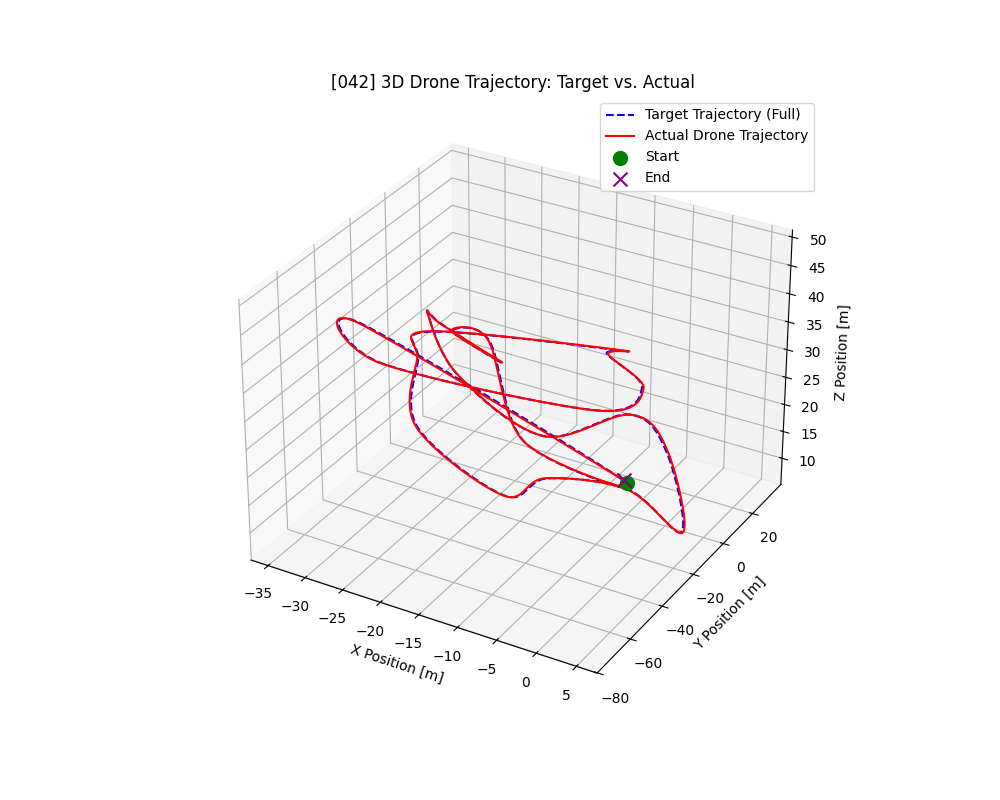}
        \caption{A sample trajectory with spiral maneuvers.}
        \label{fig:spiral_traj}
    \end{subfigure}
    \hfill
    \begin{subfigure}[b]{0.48\textwidth}
        \centering
        \includegraphics[width=\textwidth]{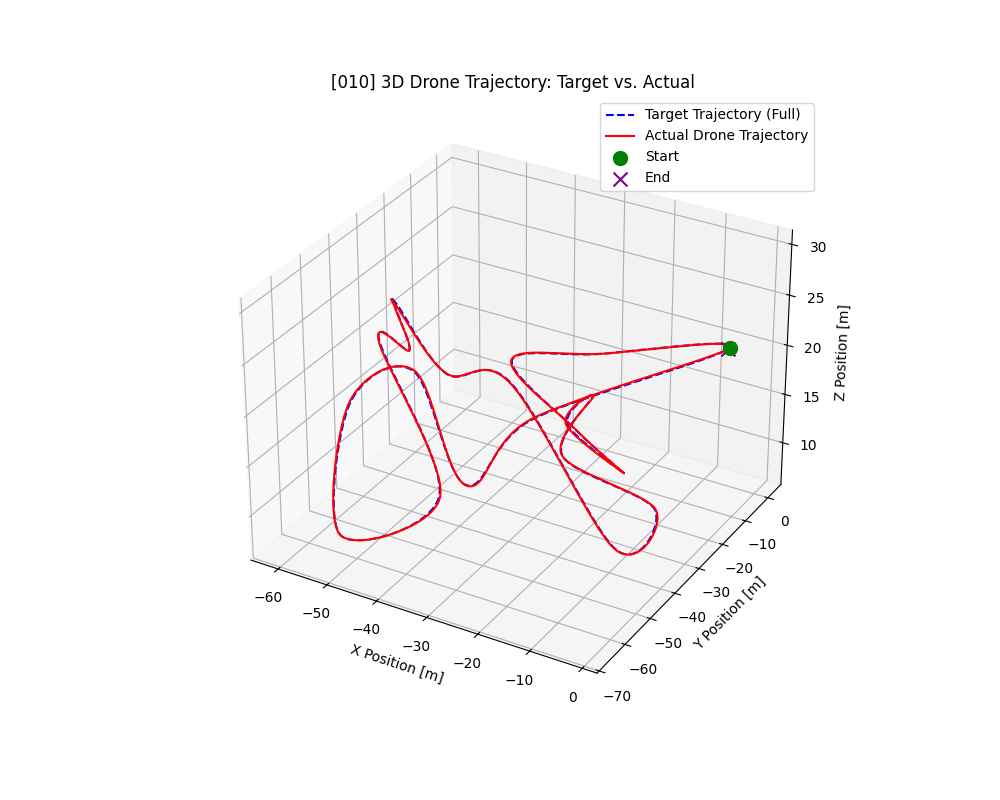}
        \caption{A sample trajectory with random wiggles.}
        \label{fig:wiggle_traj}
    \end{subfigure}
    \caption{Sample aggressive trajectories generated in the PyBullet simulator.}
    \label{fig:sample_trajectory}
\end{figure}

\paragraph{State-Dependent Noise Model.}
To simulate realistic sensor characteristics, the ground-truth velocity is corrupted by zero-mean Gaussian noise, $\mathbf{v}_{\mathrm{noisy}} \sim \mathcal{N}(\mathbf{v}_{\mathrm{gt}},\,C_v)$. The covariance matrix $C_v$ is state-dependent, designed to scale with the drone's speed. The standard deviation $\sigma_v$ for each velocity axis is computed using a sigmoid function of the velocity magnitude $\|\mathbf{v}_{\mathrm{gt}}\|$:
\begin{equation}
 \sigma_v(\|\mathbf{v}_{\mathrm{gt}}\|) = \sigma_{\min} + (\sigma_{\max} - \sigma_{\min}) \cdot \frac{1}{1 + \exp(-\lambda (\|\mathbf{v}_{\mathrm{gt}}\| - v_{\text{mid}}))},
\end{equation}
where the variance on each axis is $\sigma_v^2$. We set the minimum and maximum standard deviations to $\sigma_{\min}=0.2$\,m/s and $\sigma_{\max}=1.0$\,m/s, respectively. The steepness $\lambda$ is set to 0.8, and the midpoint velocity $v_{\text{mid}}$ is dynamically adjusted based on the estimated average speed of each trajectory to ensure a realistic noise profile.

\subsection{Baseline and Model Implementation Details}
\label{app:impl}

We compare ReLN against two baseline classes chosen to isolate the effect of geometric priors.

\paragraph{Non-equivariant baselines.}  We use a standard 1D ResNet architecture with temporal convolutional blocks that processes flattened input sequences. The \textbf{ResNet (velocity-only)} model receives only the 3D velocity vector. The \textbf{ResNet (velocity + covariance)} model receives the flattened $3\times3$ covariance matrix concatenated to the velocity vector. 

\paragraph{Eigendecomposition-based {$\mathbf{SO}\textbf{(3)}$}-Equivariant Baseline.} This model adapts the 1D ResNet backbone for $\mathrm{SO}(3)$ equivariance using VN layers. Since VNs cannot directly ingest matrices, we decompose each covariance matrix $C = V \Lambda V^\top$ and use a dual-stream design:
\begin{itemize}
  \item an \emph{equivariant} stream \(\mathcal{F}_{\mathrm{eq}}=\{\mathbf v,\mathbf e_1,\mathbf e_2,\mathbf e_3\}\) comprising the measured velocity \(\mathbf v\) and the three orthonormal eigenvectors \(\mathbf e_i\), which together capture all directional information. This stream is handled by the VNs backbone.
  \item an \emph{invariant} stream \(\mathcal{F}_{\mathrm{inv}}=\{\lambda_1,\lambda_2,\lambda_3\}\) processes the corresponding eigenvalues $\{\lambda_1, \lambda_2, \lambda_3\}$, which encode orientation-independent scale information, using a standard MLP.
\end{itemize}
The two latent features from both streams are fused at the final output layer. Eigenvector ambiguities (sign or multiplicities) are resolved via a deterministic, rotation-equivariant canonicalization.

\paragraph{Spherical Harmonic (SH)-based Equivariant Baselines (TFN and SE(3)-Transformer).}

We evaluate two representative $\mathrm{SO}(3)$-equivariant graph architectures built from irrep-valued features and steerable bases expanded in spherical harmonics (SH): Tensor Field Networks (TFN)~\citep{thomas2018tensor} and the SE(3)-Transformer~\citep{fuchs2020se}. 
Both enforce equivariance by restricting linear maps to $\mathrm{SO}(3)$ intertwiners, implemented through SH steerable kernels and tensor-product (Clebsch--Gordan) couplings.

\textit{Transformation Rules and Irrep Decomposition.}
\label{app:irrep_decomp}
Under a rotation $R\in \mathrm{SO}(3)$, the velocity transforms as $\mathbf v \mapsto R\mathbf v$, while a covariance transforms by congruence $C \mapsto RCR^\top$.
We represent the symmetric rank-2 covariance via the standard irrep decomposition
$\mathrm{Sym}^2(\mathbb{R}^3)\cong L{=}0 \oplus L{=}2$,
i.e., $C = \tfrac{\mathrm{tr}(C)}{3}I + C_0$ with $\mathrm{tr}(C)$ as a type-$0$ scalar and $C_0 := C - \tfrac{\mathrm{tr}(C)}{3}I$ as a type-$2$ traceless-symmetric tensor.

\textit{Feature construction and temporal graph.}
Each timestamp is encoded as an $\mathrm{SO}(3)$ fiber with degrees $L\in\{0,1,2\}$: $\mathrm{tr}(C)$ for $L{=}0$, $\mathbf v$ for $L{=}1$, and $C_0$ for $L{=}2$.
We form a temporal chain graph connecting $t$ to $t\pm 1$ and drive message passing using edge attributes derived from relative offsets.
TFN aggregates messages via steerable kernel filtering (with learned radial profiles), whereas SE(3)-Transformer replaces fixed kernel aggregation with multi-head, content-dependent equivariant self-attention~\citep{thomas2018tensor, fuchs2020se}.
In both cases, uncertainty is stored in separate type-$0$ and type-$2$ blocks, and interactions between them are mediated by learned equivariant tensor-product couplings rather than by treating $C$ as a single matrix feature throughout the network. 

\textit{Relation to our formulation.}
In contrast, ReLN embeds $(\mathbf v, C)$ into a unified matrix space $\mathfrak{gl}(3)$ and processes both under a single adjoint action, directly matching the congruence rule $C \mapsto RCR^\top$ in a shared Lie-algebraic backbone.

\textit{Empirical observation.}
In our experiments, TFN consistently outperformed the SE(3)-Transformer. We attribute this to the nature of drone state estimation as a path integration task. The SH-based convolutional kernels in TFN act as learned local integrators, which are better suited for capturing high-frequency temporal dynamics than the global dependency focus of attention mechanisms. While attention provides flexibility, it can introduce susceptibility to high-frequency noise in aggressive flight sequences, whereas the local inductive bias of SH-convolutions provides a more stable prior for trajectory reconstruction.

\paragraph{Reductive Lie Neurons (ReLNs).} 
The ReLN model shares a similar backbone but incorporates the \texttt{ReLN-Bracket} layer. In contrast to the VN and graph-based baselines, ReLNs provide a unified framework for velocity and covariance processing within the Lie algebra $\mathfrak{gl}(3)$. Velocities $\mathbf{v} \in \mathbb{R}^3$ are lifted into $\mathfrak{so}(3) \subset \mathfrak{gl}(3)$ via $K = \mathbf{v}^\wedge$, while covariance $C$ (or $\log C$) is treated as a structured geometric input. Both transform under the same adjoint action: $K' = RKR^\top$ and $C' = RCR^\top$. By utilizing the matrix logarithm, $\log C \in \mathrm{Sym}(3) \subset \mathfrak{gl}(3)$, we embed the manifold-valued uncertainty into a linear subspace compatible with Lie-algebraic processing. The final velocity estimate is equivariantly extracted via the projection $\tilde{\mathbf{v}} = (\tfrac{1}{2}(A - A^\top))^{\vee}$ from the network's matrix output $A \in \mathbb{R}^{3 \times 3}$.

\subsection{Training and Evaluation Protocol}

\paragraph{Problem Formulation.}
The network is trained to predict the drone's 3D position $\mathbf{p}_t \in \mathbb{R}^3$ at the end of a given time window, based on a sequence of noisy velocity measurements and their corresponding covariances within that window (e.g., a 1-second history). All models are trained by minimizing the Mean Squared Error (MSE) between the predicted position $\hat{\mathbf{p}}_t$ and the ground-truth position $\mathbf{p}_{t, \text{gt}}$. The loss function is defined as $\mathcal{L} = \| \hat{\mathbf{p}}_t - \mathbf{p}_{t, \text{gt}} \|_2^2$.

\paragraph{Dataset and Optimization.}
We partition the dataset using a standard 80:10:10 train/validation/test split. All models are trained on identical splits to ensure fair comparison. Models are optimized using the AdamW optimizer with a ReduceLROnPlateau learning rate scheduler based on validation loss.

\paragraph{Evaluation Metrics.}
We report the following pose-regression metrics over the test set:
\begin{itemize}
    \item \textbf{Absolute Trajectory Error (ATE):} The root-mean-square error (RMSE) between the ground-truth positions $p_k$ and predicted 3D positions $\hat{p}_k$ over the entire trajectory, calculated as $\sqrt{\frac{1}{N}\sum_{k=1}^N \|\hat{p}_k - p_k\|^2}$, measured in meters.
    
    \item \textbf{ATE\(_\%\):} The ATE normalized by the total trajectory length $L$, expressed as a percentage ($100 \times \text{ATE} / L$). This metric provides a scale-invariant measure of error, which is crucial for fairly comparing performance across our aggressive flight trajectories of varying lengths.
    
    \item \textbf{Relative Translation Error (RTE):} This metric evaluates local consistency by measuring the relative translational error over fixed-length sub-trajectories. For each window of duration $\Delta t = 2.0\,\text{s}$ ($n=400$ samples), we compute the RMSE of the positional difference between estimated and ground-truth relative displacements: $\|(\hat{p}_{k+n} - \hat{p}_k) - (p_{k+n} - p_k)\|$. We report the mean error by sliding this window across the trajectory with a stride of $0.5\,\text{s}$ ($100$ samples) to ensure a robust assessment of short-term stability.
\end{itemize}

To explicitly validate equivariance, we also evaluate all models on the test set after applying a set of random $\mathrm{SO}(3)$ rotations to the entire input sequence.

\subsection{Eigenvector Canonicalization for the VN Baseline}
\label{app:canonicalization}
To resolve ambiguities in the eigendecomposition $C=V\Lambda V^\top$ for the VN baseline while preserving $SO(3)$ geometry, we canonicalize the eigenvector matrix $V=[\mathbf{e}_1,\mathbf{e}_2,\mathbf{e}_3]$ using the following sequence. To avoid producing reflection matrices (where $\det V = -1$), sign disambiguation must precede the enforcement of a right-handed frame.

\begin{enumerate}
    \item \textbf{Sign Disambiguation:} For each eigenvector $\mathbf{e}_i$ associated with a distinct eigenvalue, we enforce a consistent orientation relative to the velocity vector $\mathbf{v}$ by ensuring $\mathbf{v}^\top\mathbf{e}_i \ge 0$. If $\mathbf{v}^\top\mathbf{e}_i < 0$, we set $\mathbf{e}_i \leftarrow -\mathbf{e}_i$.
    \item \textbf{Right-handed Frame:} After the individual signs are fixed, we ensure the matrix $V$ represents a proper rotation. If $\det V < 0$, we flip the sign of the third eigenvector, $\mathbf{e}_3 \leftarrow -\mathbf{e}_3$, to ensure $\det V = +1$. This step is performed last to ensure that any sign changes in the previous step do not inadvertently result in a reflection frame.
    \item \textbf{Multiplicity Handling:} In cases of repeated eigenvalues, we project the velocity vector $\mathbf{v}$ onto the corresponding eigenspace to uniquely define the first basis vector of that subspace, then complete the orthonormal basis via the Gram-Schmidt process, followed by the determinant check in Step 2.
\end{enumerate}
By applying the determinant constraint after sign alignment, we guarantee that the baseline operates on valid rotation matrices, ensuring a fair and geometrically consistent comparison with our $SO(3)$-equivariant model.

\subsection{Full Experimental Results and Ablation Study for Drone State Estimation}
\label{app:drone_full_results}

Table~\ref{tab:drone_full_comparison} provides the exhaustive results for all models across three input modalities: $v, (v, C), (v, \log C)$. This comprehensive comparison highlights that while the $\log$-covariance interface generally improves performance for most equivariant models, our ReLN architecture remains the superior backbone due to its ability to jointly model the full reductive structure.

\begin{table*}[h]
\centering
\small % 폰트 크기를 키워 가독성 향상
\caption{Complete performance report on the drone dataset including all input variants and ablations. We report ATE (meters), ATE$_{\%}$ (scaled), and RTE (meters). \textbf{SO(3)} denotes evaluation under random test-time rotations.}
\label{tab:drone_full_comparison}
\begin{tabular*}{\textwidth}{@{\extracolsep{\fill}} l c ccc ccc}
\toprule
\multirow{2}{*}{\textbf{Model}} & \multirow{2}{*}{\textbf{Input}} & \multicolumn{3}{c}{\textbf{ID (In-Distribution)}} & \multicolumn{3}{c}{\textbf{SO(3) (Rotated)}} \\
\cmidrule(lr){3-5} \cmidrule(lr){6-8}
& & \textbf{ATE} $\downarrow$ & \textbf{ATE}$_{\%}$ $\downarrow$ & \textbf{RTE} $\downarrow$ & \textbf{ATE} $\downarrow$ & \textbf{ATE}$_{\%}$ $\downarrow$ & \textbf{RTE} $\downarrow$ \\
\midrule
ResNet & $v$ & 208.07 & 95.06 & 107.60 & 217.02 & 100.39 & 111.29 \\
ResNet & $(v, C)$ & 205.11 & 94.94 & 106.07 & 213.26 & 98.90 & 109.37 \\
\midrule
Vector Neurons & $v$ & 17.36 & 7.52 & 13.51 & 17.36 & 7.52 & 13.51 \\
Vector Neurons & $(v, C)$ & 191.78 & 88.66 & 98.39 & 190.22 & 88.47 & 98.26 \\
\midrule
Tensor Field Network & $v$ & 24.59 & 10.95 & 18.23 & 24.59 & 10.95 & 18.23 \\
Tensor Field Network & $(v, C)$ & 17.56 & 7.60 & 14.40 & 17.56 & 7.60 & 14.40 \\
Tensor Field Network & $(v, \log C)$ & 16.83 & 7.56 & 13.34 & 16.83 & 7.56 & 13.34 \\
\midrule
SE(3)-Transformer & $v$ & 22.22 & 9.85 & 17.63 & 22.22 & 9.85 & 17.63 \\
SE(3)-Transformer & $(v, C)$ & 21.67 & 9.37 & 16.77 & 21.67 & 9.37 & 16.77 \\
SE(3)-Transformer & $(v, \log C)$ & 20.12 & 8.84 & 15.36 & 20.12 & 8.84 & 15.36 \\
\midrule
\textit{Ablations (Ours)} \\
Lie Neurons$^\dagger$ & $(v, C)$ & 16.86 & 7.43 & 13.65 & 16.86 & 7.43 & 13.65 \\
Lie Neurons$^\dagger$ & $(v, \log C)$ & 15.65 & 6.76 & 12.04 & 15.65 & 6.76 & 12.04 \\
ReLN (no semisimple) & $(v, \log C)$ & 16.27 & 7.00 & 12.65 & 16.27 & 7.00 & 12.65 \\
\midrule
\textit{Our Equivariant Model} \\
ReLN & $v$ & 16.85 & 7.31 & 12.70 & 16.85 & 7.31 & 12.70 \\
ReLN & $(v, C)$ & 16.49 & 7.21 & 13.02 & 16.49 & 7.21 & 13.02 \\
ReLN & $(v, \log C)$ & \textbf{13.92} & \textbf{5.99} & \textbf{11.04} & \textbf{13.92} & \textbf{5.99} & \textbf{11.04} \\
\bottomrule
\end{tabular*}
\begin{flushleft}
\scriptsize
$^\dagger$ The Lie Neurons architecture~\citep{lin2023lie} is mathematically equivalent to the ``no-center'' variant of ReLN, i.e., ReLN with the bilinear form restricted to the semisimple ideal.
\end{flushleft}
\end{table*}

\begin{figure*}[h]
    \centering
    \includegraphics[width=\linewidth]{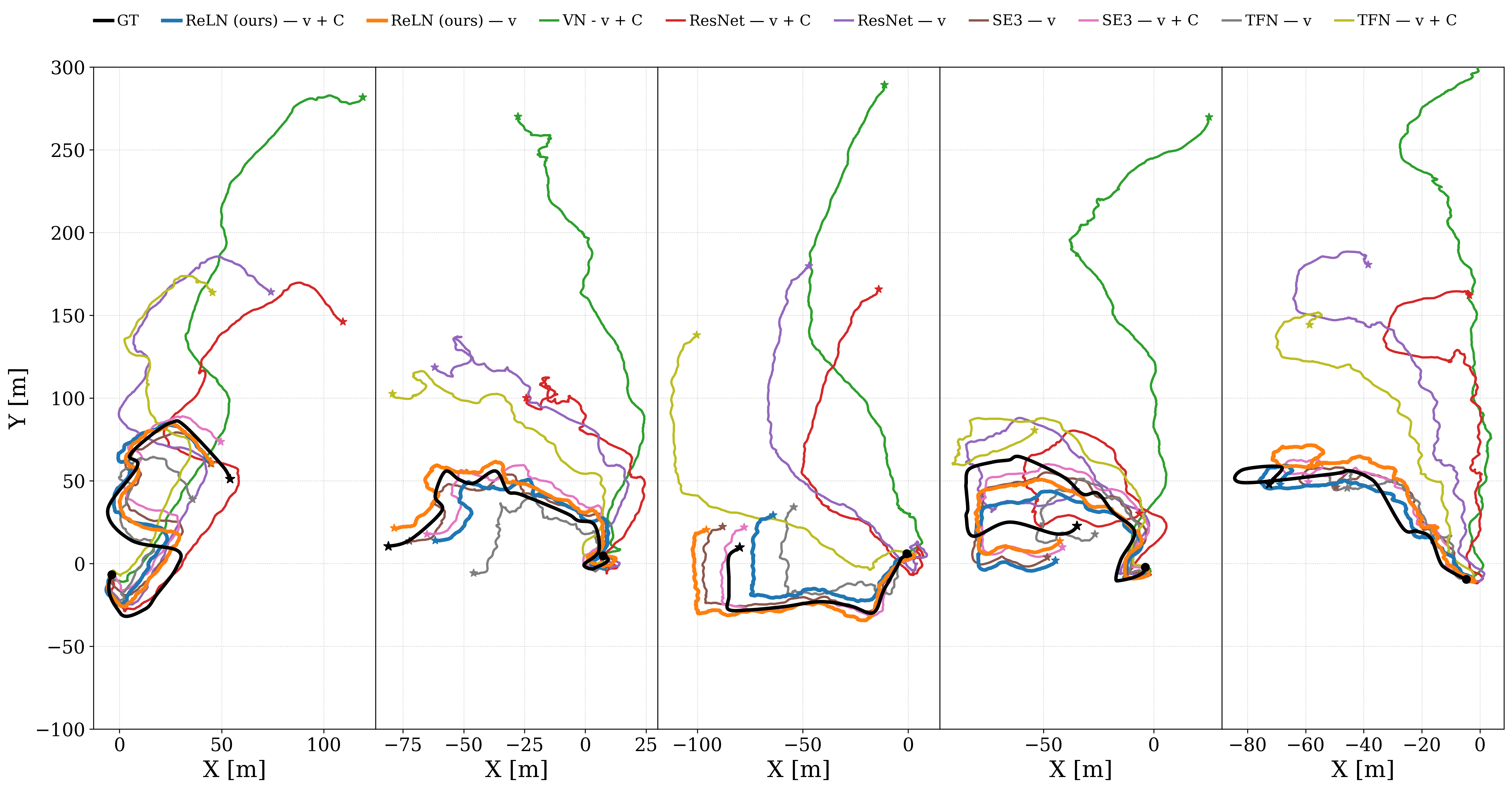}
    \caption{{Qualitative comparison of reconstructed trajectories} ReLN variants  maintain superior tracking fidelity and stability compared to non-equivariant  and standard equivariant baselines.}
    \label{fig:drone_qualitative_full}
\end{figure*}

\paragraph{Qualitative Analysis of Trajectory Fidelity.}
Figure~\ref{fig:drone_qualitative_full} validates the qualitative robustness of ReLNs under high-dynamic flight conditions. While standard MLP-based ResNets and eigendecomposition-based Vector Neurons suffer from rapid error accumulation and structural divergence, ReLN maintains high fidelity to the ground truth even in velocity-only configurations. 

The integration of covariance, $(v, C)$, and specifically the \textbf{$\log$-covariance} provides effective regularization via uncertainty-aware weighting. Notably, while steerable baselines such as TFN and SE(3)-Transformers show improved stability with $\log$-covariance, they still exhibit observable deviations during sharp maneuvers where geometric consistency is critical. This confirms our finding that the reductive Lie-algebraic backbone offers a more stable and geometrically principled space for fusing heterogeneous sensor signals and their associated uncertainties.

\section{Proof of $\mathrm{SO}(3)$-Equivariance for ReLN Velocity Extract with Covariance Inputs}
\label{app:equivariance_proof}

This section provides a formal proof for the $\mathrm{SO}(3)$-equivariance of our Reductive Lie Neuron (ReLN) architecture when processing a velocity vector and a covariance matrix. We first establish the foundations for processing covariance matrices within a Lie-algebraic framework and then present the main proof.

\subsection{$\mathrm{SO}(3)$-Equivariant Vector Extraction via Skew-Symmetric Projection}

Our network, $\Phinet$, is designed to be adjoint-equivariant. It maps geometric inputs—such as an embedded velocity $K \in \mathfrak{so}(3)$ and a covariance matrix $S \in \mathrm{SPD}(3)$—to a matrix feature $A \in \mathbb{R}^{3\times3}$. The inputs transform under the adjoint action of any rotation $R \in \mathrm{SO}(3)$:
\begin{equation}
    K' = \mathrm{Ad}_R(K) = RKR^\top, \quad S' = \mathrm{Ad}_R(S) = RSR^\top.
    \label{eq:appendix_input_transform}
\end{equation}
By construction, the network's output feature $A$ transforms according to the same law:
\begin{equation}
    \Phinet(K', S') = \mathrm{Ad}_R\big(\Phinet(K, S)\big) = R\,\Phinet(K, S)\,R^\top.
    \label{eq:appendix_net_equivariance}
\end{equation}
To obtain the final 3D velocity vector, we project the output matrix $A$ onto its skew-symmetric component and apply the vee operator. The following proposition formalizes the equivariance of this extraction mechanism.

\begin{proposition}[Equivariance of Skew-Symmetric Extraction]
Let a network $\Phinet$ and its inputs transform according to Eqs. \ref{eq:appendix_input_transform} and \ref{eq:appendix_net_equivariance}. If a vector $\tilde{\mathbf{v}} \in \mathbb{R}^3$ is extracted from the output matrix $A = \Phinet(K, S)$ via the projection
\begin{equation}
    A_{\mathrm{skew}} = \tfrac{1}{2}(A - A^\top), \qquad \tilde{\mathbf{v}} = (A_{\mathrm{skew}})^\vee,
    \label{eq:appendix_extraction}
\end{equation}
then the vector $\tilde{\mathbf{v}}'$ extracted from the transformed output $A' = \Phinet(K', S')$ transforms covariantly as $\tilde{\mathbf{v}}' = R\tilde{\mathbf{v}}$.
\end{proposition}

\begin{proof}
By the adjoint-equivariance property in Eq.~\ref{eq:appendix_net_equivariance}, the network satisfies $\Phinet(RKR^\top, RSR^\top) = R\,\Phinet(K, S)\,R^\top = RAR^\top$. Let $A' = RAR^\top$. The skew-symmetric component of the transformed output $A'$ is:
\begin{equation}
\begin{split}
    A'_{\mathrm{skew}} &= \tfrac{1}{2}(A' - A'^\top) \\
    &= \tfrac{1}{2}\big(RAR^\top - (RAR^\top)^\top\big) \\
    &= \tfrac{1}{2}\big(RAR^\top - RA^\top R^\top\big) \\
    &= R \left( \tfrac{1}{2}(A - A^\top) \right) R^\top \\
    &= R A_{\mathrm{skew}} R^\top = \mathrm{Ad}_R(A_{\mathrm{skew}}).
\end{split}
\end{equation}
The vee map, $(\cdot)^\vee: \mathfrak{so}(3) \to \mathbb{R}^3$, is itself an equivariant map satisfying $(\mathrm{Ad}_R(X))^\vee = R\,(X^\vee)$ for any $X \in \mathfrak{so}(3)$. Applying this property yields the desired result:
\begin{equation}
    \tilde{\mathbf{v}}' = (A'_{\mathrm{skew}})^\vee = (\mathrm{Ad}_R(A_{\mathrm{skew}}))^\vee = R\,(A_{\mathrm{skew}})^\vee = R\tilde{\mathbf{v}}.
\end{equation}
\end{proof}

\begin{remark}
The proof relies on three properties: (i) both inputs transform under the adjoint action $X \mapsto RXR^\top$; (ii) the network $\Phinet$ is equivariant to this action; and (iii) the output is projected onto $\mathfrak{so}(3)$ before the vee operator is applied. As established previously, these conditions hold whether the network ingests the raw covariance $S$ or its logarithm $\log S$.
\end{remark}

%%%%%%%%%%%%%%%%%%%%%%%%%%%%%%%%%%%%%%%% 3D GS %%%%%%%%%%%%%%%%%%%%%%%%%%%%%%%%%%%%%%%% 

\section{Details on 3D Gaussian Splatting Experiments}
\label{app:3dgs_details}

In this section, we detail the experimental protocols and architectural modifications for the 3D Gaussian Splatting (3DGS) experiments. Figure~\ref{fig:framework_overview} provides a comprehensive overview of the ReLN-integrated Gaussian-MAE architecture, highlighting the equivariant data flow from geometric lifting to projection. We provide a rigorous theoretical justification for our architectural design to achieve \textbf{end-to-end equivariance} by jointly learning geometric and photometric features within a unified Lie-algebraic framework.

\subsection{Datasets and Protocol}
\label{app:3dgs_datasets}
We adopt a standard two-stage protocol common in self-supervised learning for 3D representation~\citep{ma2025large}, separating the pre-training domain from the downstream evaluation task to assess generalization.

\paragraph{Pre-training: ShapeSplat (ShapeNet-derived).}
We utilize \textbf{ShapeSplat}~\citep{ma2025large}, a large-scale collection of object-level 3D Gaussian primitives derived from ShapeNet. Each object is represented by a set of anisotropic Gaussians $\Theta_i=\{\mu_i,\Sigma_i,\alpha_i,c_i\}_{i=1}^{N}$. Following the baseline protocol, we sample $N=1,024$ Gaussians via furthest point sampling (FPS) and apply a $60\%$ masking ratio for the masked autoencoding objective.

\paragraph{Downstream: ModelNet10 Classification.}
To evaluate rotational robustness, we fine-tune the pre-trained encoder on \textbf{ModelNet10}. We report accuracy on:
(i) \textbf{Aligned}: The canonical test set.
(ii) \textbf{Rotated}: The test set under random rotations $R \sim SO(3)$, testing whether the learned representations remain stable under arbitrary orientations.

\subsection{Holistic Equivariance via Joint Feature Learning}
\label{app:3dgs_holistic}
A critical limitation of the baseline architecture~\citep{ma2025large} is the naive concatenation of heterogeneous Gaussian attributes. A 3D Gaussian combines \emph{equivariant} quantities ($\mu, \Sigma$) with \emph{invariant} quantities ($c, \alpha$). Our goal in ReLN is to construct a \textbf{holistic equivariant architecture} where these distinct attribute types are learned jointly without violating the symmetry of the feature space.

\paragraph{1. The Equivariant Component (Active Geometry).}
The mean position $\mu \in \mathbb{R}^3$ and covariance $\Sigma \in \mathbb{R}^{3\times3}$ dictate the spatial structure. Under a global rotation $R \in SO(3)$, these transform via $\mu \mapsto R\mu$ and $\Sigma \mapsto R\Sigma R^\top$.
Recognizing that the covariance update corresponds to the adjoint action $\mathrm{Ad}_R(\Sigma) = R\Sigma R^{-1}$, we embed $\Sigma$ into the symmetric subspace of $\mathfrak{gl}(3)$. This ensures that the primary geometric features are processed within a consistent $\mathrm{Ad}$-equivariant vector space throughout the network.

\paragraph{2. The Invariant Component (Photometry and Geometric Norms).}
To support robust representation learning, we augment the state space with a set of \textbf{$O(3)$-invariant scalars (Type-0 features)}. This set includes:
\begin{itemize}
    \item \textbf{Intrinsic Attributes.}
Opacity $\alpha$ and color $c$ are treated as rotation-invariant channels in our setup.
Following the baseline, we retain only the \emph{zeroth-order} spherical-harmonic term ($l=0$) for color (three coefficients for the RGB channels).
Concretely, an SH color field can be written as
$
f(\hat{\mathbf d})=\sum_{l,m} c_{lm}\,Y_l^m(\hat{\mathbf d}),
$
and for $l=0$ this reduces to
$
f(\hat{\mathbf d}) = c_{00}\,Y_0^0(\hat{\mathbf d}).
$
Since $Y_0^0$ is constant on the sphere,
the $l=0$ coefficient is unchanged by rotations (equivalently, the Wigner-$D$ matrix satisfies $D^{0}(R)=1$).
Therefore, the RGB $l=0$ coefficients are best viewed as \emph{three independent scalar channels}
(i.e., multiplicity-$3$ copies of the $L=0$ irrep), and are invariant under object rotations:
$c \mapsto c$ and $\alpha \mapsto \alpha$.

    \item \textbf{Derived Geometric Invariants:} Crucially, we also explicitly compute rotation-invariant geometric descriptors, such as the norms of position vectors ($\|\mu\|$) and rotation/scale magnitudes derived from $\Sigma$.
\end{itemize}
In the ReLN framework, treating these as Type-0 features serves to \textit{facilitate its feature integration}. These scalars participate in the learning process by modulating the equivariant features without violating the underlying group structure and equivariance.

\paragraph{3. Joint Learning in ReLN Attention.}
The explicit stratification of these streams is a structural prerequisite for \textbf{joint equivariant learning}. By assigning each attribute to its correct algebraic type (Type-1/Type-2 for geometry, Type-0 for scalars), the ReLN attention mechanism can validly integrate information across all modalities. The invariant scalars ($c, \alpha, \|\mu\|$) contribute to the computation of attention scores (invariant inner products), thereby modulating how the equivariant geometric features ($\mu, \Sigma$) are aggregated. This design ensures that the entire network optimizes a single, cohesive objective function while guaranteeing that the final representation remains equivariant end-to-end.

\subsection{Architectural Modifications}
\label{app:3dgs_arch}
To enforce this holistic equivariance, we replaced non-robust components of the baseline architecture with ReLN equivalents that respect the defined type system.

\paragraph{ReLNEncoder and Geometric Aggregation.}
Our \textbf{ReLNEncoder} begins with a \textit{Geometric Lifting} stage that maps raw Gaussian parameters ($C, S, R$) into $\mathfrak{gl}(3)$. The baseline uses a `SoftEncoder` (MLP on raw coordinates) which breaks equivariance because standard linear weights $W$ do not rotate with the input frame. We replaced this with \textbf{ReLNEncoder}, which employs $\mathrm{Ad}$-equivariant linear layers ($\texttt{ReLNLinear}$) to process local neighborhoods. ReLNEncoder incorporates a \textit{Geometric Aggregator}, which utilizes an equivariant norm-based max pooling mechanism to compress grouped splats into structured $C' \times 9$ tokens. This ensures that the fundamental geometric tokens are processed within a mathematically consistent vector space.

\paragraph{Central Positional Embedding.}
To provide spatial context without violating symmetry, we process group centroids ($\mu$) through a dedicated \textit{Central Positional Embedding}. This module lifts 3D spatial coordinates into a structured Lie-algebraic representation, which is then injected into both the encoder and decoder stages (as illustrated in Figure~\ref{fig:framework_overview}). This injection strategy ensures that the transformer blocks distinguish relative spatial relations while maintaining rotational equivariance across the entire processing pipeline.

\paragraph{Invariant Global Pooling (vs. CLS Token).}
For downstream classification tasks, we eschew the conventional learnable \texttt{[CLS]} token, recognizing that a fixed-reference-frame parameter \textbf{inherently imposes} a canonical orientation bias that violates rotational symmetry. Instead, we employ a permutation-invariant \textit{Global Max Pooling} over the final sequence of equivariant tokens. This ensures that the aggregated global representation remains stable and rotates consistently with the input geometry, thereby satisfying holistic rotational symmetry for robust 3D representation learning.

\subsection{Training and Evaluation Protocol}
\label{app:3dgs_training}

\textbf{Pre-training Setup.} 
We pre-trained both the baseline and ReLN models for 300 epochs on the ShapeSplat (ShapeNet) dataset. We utilized the AdamW optimizer with an initial learning rate of $1 \times 10^{-3}$, a weight decay of $0.05$, and a cosine annealing scheduler with a 10-epoch warm-up. For the ReLN model, we explicitly disabled the \texttt{soft\_knn} option to prevent the leakage of non-equivariant features from the legacy encoder implementation.

\textbf{Fine-tuning Setup.} 
For the downstream ModelNet10 classification task, we fine-tuned the pre-trained encoder for an additional 300 epochs. In this phase, we adjusted the learning rate to $5 \times 10^{-4}$ while maintaining the same weight decay ($0.05$) and cosine scheduling strategy.

\textbf{Evaluation Metrics.} To rigorously evaluate geometric stability, we construct two test benchmarks:
\begin{itemize}
    \item \textbf{Standard (Aligned):} The original ModelNet10 test set where objects are canonically aligned.
    \item \textbf{Rotated (SO(3)):} We apply a random rotation $R \in \mathrm{SO}(3)$ to every object. For a Gaussian splat, we transform parameters as $\mu' = R\mu$, $\Sigma' = R \Sigma R^\top$.
\end{itemize}

\subsection{Additional Convergence Results}
Figure~\ref{fig:loss_convergence} presents the detailed training dynamics. The ReLN model consistently achieves lower loss values across all geometric attributes—Rotation, Scale, and Spherical Harmonics (SH)—demonstrating that enforcing equivariance stabilizes the learning of intrinsic 3D shape descriptors.

\begin{figure*}[h]
    \centering
    \includegraphics[width=0.95\linewidth]{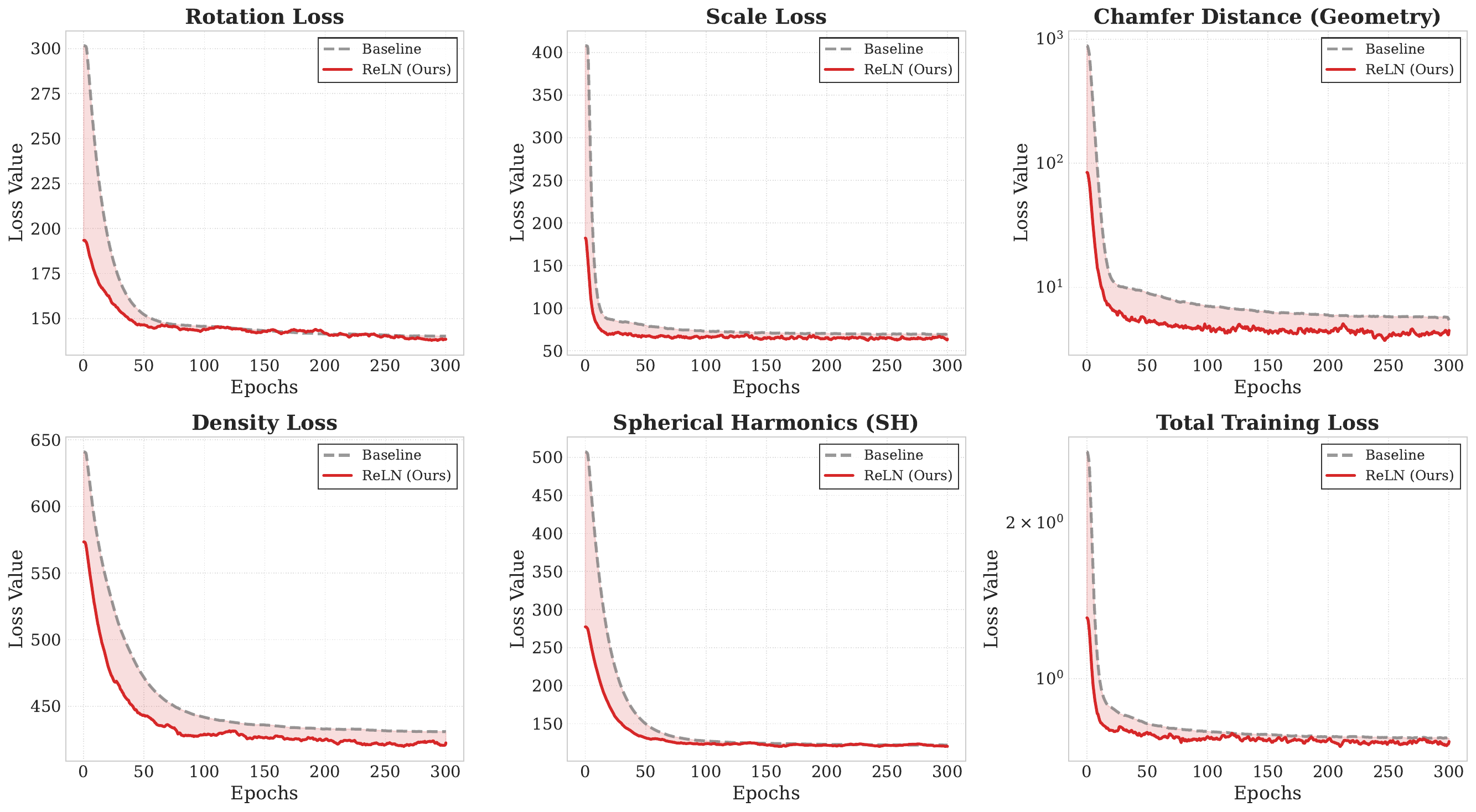}
    \caption{\textbf{Pre-training Dynamics on 3D Gaussian Splats.} Comparison of training convergence between the Baseline~\citep{ma2025large} and ReLN (Ours) across 300 epochs on \textbf{ShapeNet}. The red shaded regions highlight the gap where ReLN achieves lower reconstruction error than the baseline. ReLN consistently outperforms the baseline across \textbf{all} tracked metrics (including Rotation, Scale, Density, and Chamfer Distance), demonstrating that equivariance improves every aspect of 3D Gaussian reconstruction.}
    \label{fig:loss_convergence}
\end{figure*}

\section{Experimental Details for the Double-Pendulum Benchmark}
\label{app:double_pendulum_details}

To ensure a rigorous and fair comparison, we adhere to the experimental protocol and system parameters of the EMLP benchmark~\citep{finzi2021practical}.

\paragraph{System dynamics and dataset generation.}
We learn Hamiltonian dynamics of a 3D double spring pendulum exhibiting $O(2)$ symmetry about the $z$-axis.
All physical constants are set to unity ($m_i=k_i=l_i=1.0$) with gravity $g=[0,0,1]$.
Ground-truth trajectories are generated by Hamiltonian integration over a horizon of $T=30$s.
We construct $1500$ trajectory chunks of duration $1$s (sampled at five $0.2$s intervals) and split them into
$500/500/500$ for training/validation/testing.

\paragraph{Models and training.}
All compared Hamiltonian models parameterize a scalar Hamiltonian $H_\theta(z)$ and use the standard HNN update
$\dot z = J\nabla_z H_\theta(z)$, where $z$ stacks generalized coordinates and momenta and $J$ is the canonical symplectic matrix.
Both EMLP-HNN and ReLN-HNN use three hidden layers with channel width $c=128$, and are trained for $2000$ epochs with Adam
(lr=$3\times 10^{-3}$, batch size=$500$).
Rollout performance is evaluated by the geometric mean of relative errors over the full $30$s horizon, averaged over three independent trials (random seeds).

\paragraph{EMLP-HNN baseline (equivariant layer construction).}
Let $G=O(2)$ be the symmetry group acting on the $(x,y)$ plane (with $z$ unchanged), e.g.,
$R(\phi)=\mathrm{diag}(R_{2\times2}(\phi),1)$ (and similarly for reflections), and let $\rho(\cdot)$ denote the induced representation on the state space.
Equivariance of a linear map $W:V_{\mathrm{in}}\!\to\!V_{\mathrm{out}}$ requires
\begin{equation}
\rho_{\mathrm{out}}(g)\,W \;=\; W\,\rho_{\mathrm{in}}(g), \qquad \forall g\in G.
\end{equation}
Following~\citep{finzi2021practical}, EMLP enforces this by constraining the vectorized weights $v=\mathrm{vec}(W)$ under the product representation
$\rho=\rho_{\mathrm{out}}\otimes\rho_{\mathrm{in}}^{*}$:
\begin{equation}
\rho(g)\,v \;=\; v, \qquad \forall g\in G.
\end{equation}
To avoid infinitely many constraints, the condition is reduced to a finite generator set consisting of Lie algebra generators $\{A_i\}$ and discrete generators $\{h_j\}$:
\begin{equation}
d\rho(A_i)\,v = 0
\quad \text{and} \quad
(\rho(h_j)-I)\,v = 0,
\label{eq:emlp_constraints}
\end{equation}
which are assembled into a linear system $Cv=0$.
An equivariant basis $Q$ is obtained by computing a nullspace of $C$ (e.g., via SVD/Krylov methods as in the reference implementation),
and the layer is parameterized as $v=Q\beta$ during training (no further constraint solving during SGD).
This basis computation is group- and representation-specific (e.g., $O(2)$ vs.\ $SO(2)$).

\paragraph{EMLP nonlinearity.}
We use the standard EMLP gated nonlinearity~\citep{finzi2021practical}, where scalar channels produce gates that multiplicatively modulate non-scalar feature blocks,
preserving equivariance while enabling expressive nonlinear mixing.

\paragraph{ReLN-HNN (ours).}
In contrast, ReLN-HNN bypasses numerical constraint solving and constructs equivariant/invariant primitives in closed form within a reductive Lie-algebraic space (here $\mathfrak{gl}(3)$),
using bilinear/Killing-form-based operations to obtain symmetry-consistent features without group-specific basis recomputation.
The same architecture is used across different symmetry settings (e.g., $O(2)$, $SO(2)$, $D_6$) without re-solving linear constraints.

\paragraph{Computational efficiency measurement.}
Per-step FLOPs in Table~\ref{tab:efficiency} are measured using JAX HLO cost analysis, which counts device-executed floating-point operations including internal tensor contractions and projections.

\end{document}